\documentclass[twoside,11pt]{article}

\usepackage{blindtext}

\usepackage[abbrvbib, preprint]{jmlr2e}

\usepackage{booktabs} 
\usepackage{pifont}
\usepackage[caption=false]{subfig}
\usepackage{multirow}

\usepackage{lastpage}
\jmlrheading{xx}{xxxx}{1-\pageref{LastPage}}{x/xx; Revised x/xx}{x/xx}{xx-xxxx}{Zhenhan Fang, Aixin Tan and Jian Huang}
\usepackage{amsmath}
\usepackage{xcolor}
\usepackage{algorithm}
\usepackage{algpseudocode}

\newcommand{\orange}[1]{{\color{orange}{#1} }}

\firstpageno{1}

\begin{document}

\title{TRACE: Transport Alignment Conformal Prediction via Diffusion and Flow Matching Models} 

\author{\name Zhenhan Fang \email zhenhan-fang@uiowa.edu \\
       \addr Department of Statistics and Actuarial Science\\
       University of Iowa\\
       Iowa, IA 52242, USA
       \AND
       \name Aixin Tan \email aixin-tan@uiowa.edu \\
       \addr Department of Statistics and Actuarial Science\\
       University of Iowa\\
       Iowa, IA 52242, USA
        \AND
       \name Jian Huang \email j.huang@polyu.edu.hk \\
       \addr Departments of Data Science and AI, and Applied Mathematics \\
       The Hong Kong Polytechnic University\\
       Hong Kong, China}

\editor{xxx}

\maketitle

\begin{abstract}
Constructing valid and informative conformal prediction regions for multi-dimensional outputs remains a fundamental challenge. While conformal prediction provides finite-sample, distribution-free coverage guarantees, its practical performance critically depends on the choice of nonconformity score. Existing approaches often rely on restrictive geometric assumptions or require explicit likelihood evaluation and invertible transformations, limiting their applicability in complex generative settings.

In this work, we introduce TRACE (TRansport Alignment Conformal Estimation), a conformal prediction framework that defines nonconformity through transport alignment in diffusion and flow matching models. Rather than evaluating likelihoods, we measure how well a candidate output aligns with the learned generative dynamics by averaging denoising or velocity-matching errors along stochastic transport trajectories. The resulting transport-based scores are scalar-valued and can be calibrated using split conformal prediction, yielding valid marginal coverage under exchangeability. We further analyze the statistical properties of the proposed scores and 
their sensitivity to computational budget. Experiments on synthetic and real datasets demonstrate valid coverage and show that the resulting regions adapt naturally to multimodal and non-convex conditional distributions.
\end{abstract}

\begin{keywords}
  distribution-free inference,
  generative models, multi-dimensional outputs,  nonconformity score,  uncertainty quantification
\end{keywords}

\section{Introduction}

Constructing prediction regions with reliable uncertainty guarantees is one of the central problems in machine learning.
Given an input $x$, a prediction region aims to characterize the set of outputs $y$ that are plausible under the conditional distribution $p(y \mid x)$, typically at a predefined coverage level such as $90\%$.
In many applications, especially those involving high-dimensional or structured outputs, the geometry and interpretability of such regions are as important as their statistical validity.

Conformal prediction provides a principled framework for uncertainty quantification by offering finite-sample, distribution-free marginal coverage guarantees under the sole assumption of exchangeability \citep{papadopoulos2002inductive, vovk2005algorithmic, lei2018distribution, romano2019conformalized, lei2015conformal, pmlr-v108-izbicki20a, chernozhukov2021distributional}.
By calibrating a nonconformity score on a held-out dataset, conformal methods ensure that the resulting prediction regions achieve the desired coverage regardless of model misspecification.
This robustness has made conformal prediction increasingly attractive in settings where probabilistic modeling assumptions are difficult to justify.

Despite these advantages, extending conformal prediction to multi-dimensional outputs remains challenging.
At its core, the conformal framework reduces uncertainty quantification to the problem of ordering scalar nonconformity scores.
When the output dimension is greater than one, mapping a scalar threshold back to a meaningful region in the output space is inherently ambiguous.
As a result, the effectiveness of multi-dimensional conformal prediction depends critically on the choice of nonconformity score and the geometry it induces.

Existing approaches to this problem broadly fall into three categories.
Shape-restricted methods define regions through fixed geometric templates such as hyperrectangles or ellipsoids, but cannot capture complex distributional structure.
Sample-based methods construct regions as unions of neighborhoods around generated samples, gaining flexibility at the cost of fragmentation and sensitivity to sampling.
Latent-density-based methods, such as those based on normalizing flows, leverage invertible transformations to define prediction regions through density evaluation in a latent space. This approach yields geometrically well-behaved regions but requires explicit bijective mappings with tractable Jacobians, limiting its performance in higher-dimensional and complex multimodal distributions.
We review these approaches in detail in Section~\ref{sec:related}.


To move beyond these limitations, we turn to more expressive generative frameworks. Diffusion \citep{DBLP:journals/corr/abs-2006-11239, song2021scorebased} and flow matching \citep{lipman2023flow, liu2023flow} have recently demonstrated remarkable flexibility in modeling complex high-dimensional distributions. Unlike likelihood-based approaches that rely on explicit density evaluation, these models are trained by enforcing local alignment along transport trajectories, through denoising in diffusion models and velocity matching in flow matching. Building on this perspective and the intuition from latent-density methods, we introduce a new approach to nonconformity scoring: rather than evaluating conformity only at the final output, one can exploit the full trajectory of intermediate transport dynamics, capturing distributional information inaccessible to existing methods and defining a novel class of alignment-based conformity measures.
Intuitively, if a candidate $y$ is typical under the conditional distribution $p(y \mid x)$, then perturbations of $y$ along diffusion or flow trajectories should be well explained by the model, resulting in small transport errors.
Conversely, atypical outputs are expected to violate these local transport relations and incur larger errors.

In this work, we formalize this idea by introducing TRACE
(TRansport Alignment Conformal Estimation), a framework
that derives nonconformity scores directly from diffusion
and flow-matching objectives.
The proposed scores aggregate discrepancies between model predictions and ground-truth transport relations across multiple time steps and stochastic perturbations, yielding a scalar measure of how compatible a candidate output is with the learned conditional generative process.
These scores can be seamlessly integrated into the split conformal framework, preserving finite-sample marginal coverage guarantees without requiring explicit likelihood evaluation or invertible transformations.

TRACE conformal regions adapt naturally to the geometric structure of multimodal and non-convex conditional distributions, capturing curved and disconnected components that shape-restricted methods cannot represent.
At the same time, the stochastic nature of the underlying scores introduces an additional source of variability that distinguishes them from latent-density-based approaches.
We analyze this variability and show that it decreases at a controlled rate with the Monte Carlo budget, and that moderate computational effort suffices to achieve prediction region volumes competitive with deterministic density-based methods.

The contributions of this paper are threefold.
First, we introduce TRACE, a nonconformity scoring framework
that derives scores directly from denoising and velocity-matching
objectives, requiring neither explicit likelihood evaluation,
invertible architectures, nor Jacobian computation
(Section~\ref{sec:trace}).
Second, we prove that the Monte Carlo approximation error
in TRACE scores decays at rate $O(1/\sqrt{B})$ in both
the calibrated threshold and region volume, while coverage
remains exact for any finite budget
(Theorems~\ref{thm:score_mse}--\ref{thm:threshold_rate},
Corollary~\ref{cor:volume_stability}). Third, experiments on nine dataset configurations show
that TRACE achieves the smallest or statistically
competitive prediction region volume on all datasets. TRACE-FM consistently attains the
best performance, while TRACE-Diff delivers comparable
results, both outperforming latent-density baselines
by up to two orders of magnitude on high-dimensional
inputs (Section~\ref{sec:experiments}). Code to reproduce our experiments is available at 
\url{https://github.com/ZFangUI/TRACE}.

\section{Related Work}
\label{sec:related}

\subsection{Shape-restricted conformal methods}

Early approaches to multi-dimensional conformal prediction rely on regions with fixed geometric templates.
Axis-aligned hyperrectangles \citep{MESSOUDI2021108101} and Bonferroni-style combinations of marginal intervals \citep{lei2015conformal} are among the simplest constructions, defining rectangular regions whose sides are determined by marginal nonconformity scores.
Ellipsoidal methods \citep{johnstone2021conformal, messoudi2022ellipsoidal} generalize this by using global or local covariance estimates to define Mahalanobis-distance-based regions, allowing some adaptation to correlation structure.

While computationally efficient and easy to calibrate, these methods impose strong geometric assumptions that are rarely aligned with complex conditional distributions.
In the presence of multimodality, skewness, or non-convexity, shape-restricted regions tend to include large volumes of low-density space, leading to overly conservative uncertainty sets.

\subsection{Sample-based conformal methods}

To improve geometric expressiveness, sample-based approaches use conditional generative models to construct conformal regions.
Probabilistic conformal prediction \citep[PCP;][]{wang2022probabilistic} forms regions as unions of balls centered at samples drawn from a conditional generator.
Spherically transformed directional quantile regression \citep[ST-DQR;][]{feldman2023calibrated} learns a low-dimensional latent representation via a conditional variational autoencoder with a unimodal latent distribution, constructs a convex region in latent space using directional quantile regression, samples points within this latent region, decodes them to the output space, and forms the prediction region as a union of balls around the decoded points.

These methods can approximate highly irregular distributions, including multimodal and non-convex shapes.
 However, the resulting regions are often fragmented and sensitive
to sampling variability, and their geometry depends implicitly
on the number of generated samples and the chosen neighborhood metric.

\subsection{Latent-density-based conformal methods}

A different line of work constructs conformal regions through latent representations learned by normalizing flows.
CONTRA \citep{fang2025contra} defines nonconformity scores as latent norms under a trained conditional normalizing flow, yielding prediction regions that are connected and geometrically regular by construction.
Because the score is a deterministic function of the output given a trained model, calibration exhibits zero evaluation variance.
JAPAN \citep{english2026japan} refines this approach by using the negative log-likelihood as the nonconformity score, producing density level sets that adapt to the local geometry of the learned distribution.

These methods offer several desirable properties: deterministic scores, well-controlled region geometry, and strong empirical performance when accurate invertible transformations are available.
However, they require  invertible architectures with tractable Jacobians, and the expressiveness of such architectures may become a bottleneck for high-dimensional or highly complex conditional distributions.
TRACE operates on any model trained via denoising or velocity-matching objectives, without requiring invertibility or density evaluation.
The price is stochastic score variability, which we analyze in Section~\ref{sec:theory} and show to be controllable through the Monte Carlo budget.


\section{Preliminaries}
\label{sec:preliminaries}
Before presenting the proposed scoring framework, we
briefly review split conformal prediction and the two
generative model families on which TRACE is built.
\subsection{Problem Setup and Split Conformal Prediction}

Let $(X,Y) \in \mathcal{X} \times \mathcal{Y}$ denote a random pair with unknown joint distribution.
We assume access to an exchangeable dataset
$$
\mathcal{D} = \{(x_i, y_i)\}_{i=1}^n.
$$
Our goal is to construct, for a new input $x$, a prediction region $\widehat{C}(x) \subseteq \mathcal{Y}$ satisfying the marginal coverage guarantee
\begin{equation}
\mathbb{P}\bigl( Y \in \widehat{C}(X) \bigr) \ge 1 - \alpha.
\end{equation}

In the split conformal framework, the dataset is partitioned into a training set $\mathcal{D}_{\mathrm{train}}$ and a calibration set
$$
\mathcal{D}_{\mathrm{cal}} = \{(x_i, y_i)\}_{i=1}^{n_{\mathrm{cal}}}.
$$
A predictive model is trained on $\mathcal{D}_{\mathrm{train}}$, and a scalar nonconformity score $s(x,y)$ is defined.
The empirical scores $\{s(x_i, y_i)\}_{i=1}^{n_{\mathrm{cal}}}$ are computed on the calibration set.
Let
$$
q_{1-\alpha}
=
\text{the } \left\lceil (1-\alpha)(n_{\mathrm{cal}}+1) \right\rceil\text{-th order statistic}.
$$
The conformal prediction region is defined as
\begin{equation}
\widehat{C}(x)
=
\{ y \in \mathcal{Y} : s(x,y) \le q_{1-\alpha} \}.
\label{eq:conformal_region}
\end{equation}
Under exchangeability, this construction guarantees finite-sample marginal coverage.

\subsection{Conditional Diffusion Models}

Conditional diffusion models, particularly denoising diffusion
probabilistic models
\citep[DDPM;][]{DBLP:journals/corr/abs-2006-11239},
learn to generate samples from a conditional distribution
$p(y \mid x)$ by learning to reverse a noise corruption process.

The key idea can be described in two stages.
In the \emph{forward} stage, a clean target sample $y$ is
gradually corrupted by adding Gaussian noise over $T$ discrete
time steps. At time step $t \in \{1, \dots, T\}$, the noisy
version of $y$ is
\begin{equation}
y_t
=
\sqrt{\bar{\alpha}_t}\, y
+
\sqrt{1-\bar{\alpha}_t}\, \varepsilon,
\quad
\varepsilon \sim \mathcal{N}(0,I),
\label{eq:forward_diffusion}
\end{equation}
where $\bar{\alpha}_t = \prod_{s=1}^t \alpha_s$ is the
cumulative product of a predefined noise schedule
$\{\alpha_s\}_{s=1}^T$ with $\alpha_s \in (0,1)$.
The schedule is monotonically decreasing, so that
$\bar{\alpha}_t$ is close to $1$ for small $t$
(the noisy sample $y_t$ remains near the original $y$)
and close to $0$ for large $t$
(the noise dominates and $y_t$ becomes nearly
indistinguishable from pure Gaussian noise).
In the \emph{reverse} stage, a neural network
$\hat{\varepsilon}_\theta(y_t, t, x)$ is trained to undo this
corruption: given the noisy sample $y_t$, the time index $t$,
and the conditioning input $x$, it predicts the noise
$\varepsilon$ that was added.
If the network can accurately identify what noise was injected,
then the original $y$ can be recovered by subtracting it.
By learning to denoise at every noise level from light grain
to heavy static, the model implicitly learns the full
structure of the conditional distribution $p(y \mid x)$.

To generate new samples, one starts from pure noise
$y_T \sim \mathcal{N}(0, I)$ and iteratively applies the
learned denoiser to 
reconstruct a new sample.
For notational simplicity, we treat $t$ as a discrete index;
in practice, time is often normalized to $[0, 1]$ when
provided as input to the neural network.
\begin{figure}[!t]
\centering
\begin{minipage}[a]{\textwidth}
    \centering
    \includegraphics[width=\textwidth]{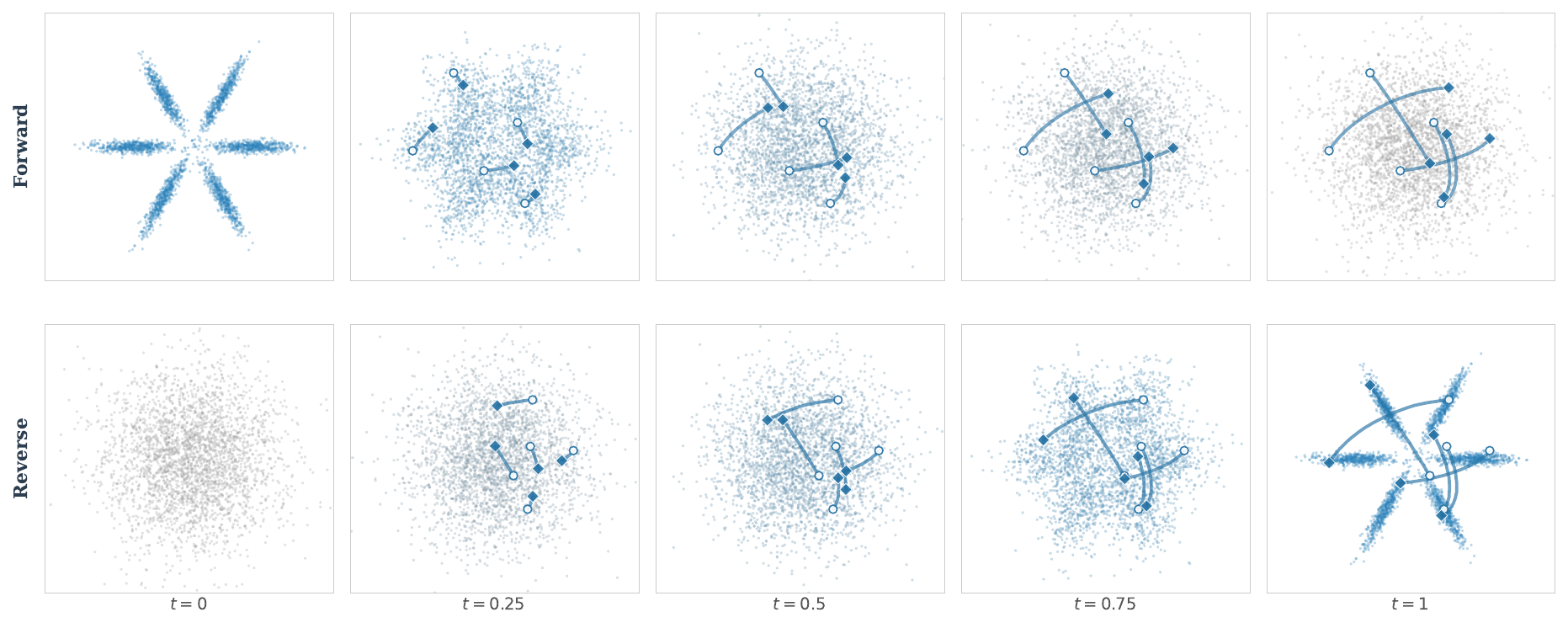}
    \centerline{\small (a) \textsc{Diffusion}}
\end{minipage}
\begin{minipage}[b]{\textwidth}
    \centering
    \includegraphics[width=\textwidth]{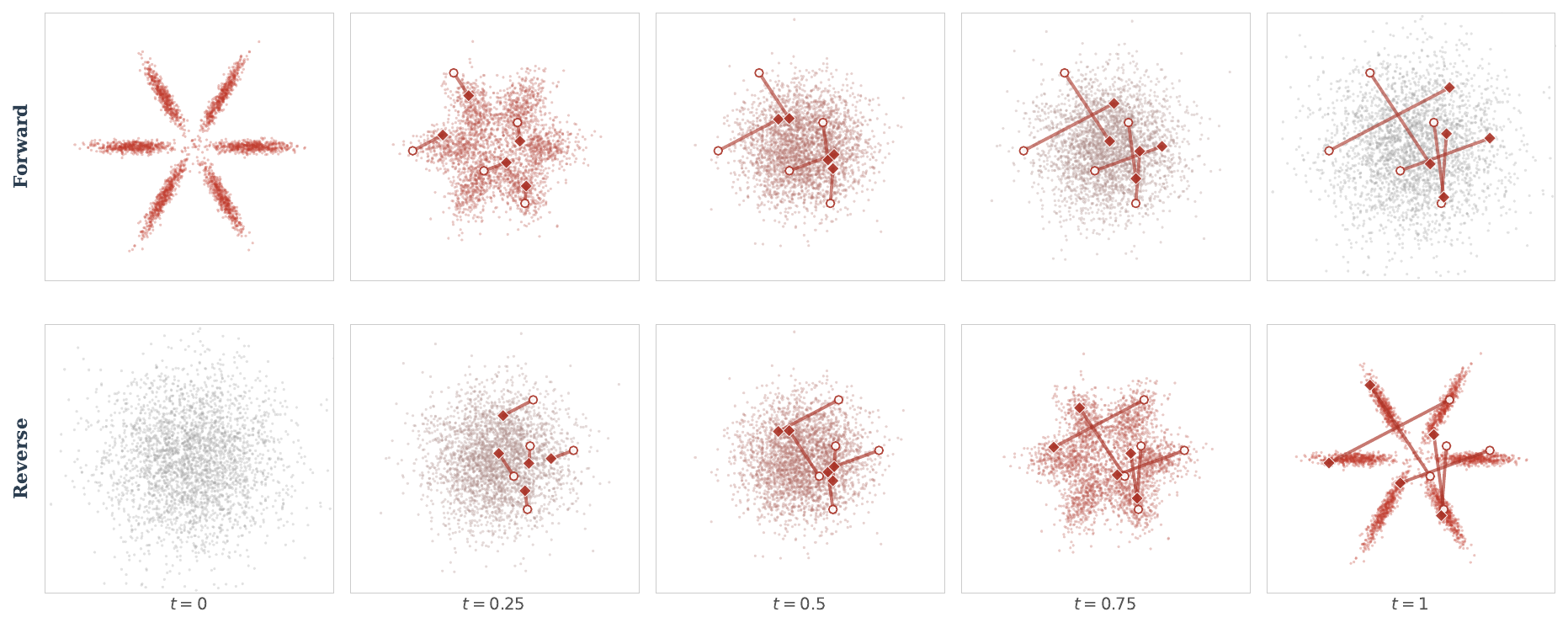}
    \centerline{\small (b) \textsc{Flow Matching}}
\end{minipage}
\caption{Forward and reverse processes of (a) DDPM and (b) Flow Matching
on a pinwheel distribution. $\circ$: trajectory origin;
$\scriptscriptstyle\blacklozenge$: current position. The reverse trajectories of DDPM
are curved due to iterative denoising, whereas those of Flow Matching
are nearly straight due to the OT path structure.}
\label{fig:model_process}
\end{figure}
\subsection{Conditional Flow Matching}

Conditional flow matching models
\citep[FM;][]{lipman2023flow, liu2023flow}
take a different approach: rather than corrupting data and
learning to reverse the corruption, they directly learn a
\emph{transport map} that moves samples from a simple
reference distribution to the conditional data distribution.

The construction is as follows. Let
$y_0 \sim \mathcal{N}(0, I)$ be a reference sample 
and $y$ a target data sample. The model defines a
straight-line path connecting them:
\begin{equation}
y_t
=
(1-t)\, y_0 + t\, y,
\quad
t \in [0,1].
\label{eq:fm_path}
\end{equation}
At $t = 0$, the path starts at the reference noise $y_0$;
at $t = 1$, it arrives at the data sample $y$.
The 
velocity needed to travel along this path is simply
$y - y_0$. A neural network $\hat{v}_\theta(y_t, t, x)$ is trained to
predict this velocity: given a point $y_t$ along the path,
the current time $t$, and the conditioning input $x$, it
outputs the direction and speed needed to continue toward
the target.
Intuitively, the model learns ``which way is the data'' from
any point along the transport path, conditioned on $x$.

To generate new samples, one draws $y_0 \sim \mathcal{N}(0, I)$
and follows the learned velocity field from $t = 0$ to $t = 1$
by solving the ordinary differential equation
$dy_t / dt = \hat{v}_\theta(y_t, t, x)$.
The resulting trajectory carries the noise sample to a point
that approximately follows $p(y \mid x)$.

Diffusion and flow matching can be viewed under the same 
stochastic interpolant framework 
\citep{albergo2025stochasticinterpolantsunifyingframework}. 
A key difference between them is the geometry of the 
transport paths. Flow matching uses linear interpolation 
between paired noise and data samples, producing nearly 
straight trajectories. In contrast, DDPM's iterative 
denoising process produces curved trajectories, as the 
denoising direction depends on the current noisy state 
at each step. This distinction is illustrated in 
Figure~\ref{fig:model_process}, and motivates the 
different nonconformity scores we develop in 
Section~\ref{sec:trace}.


\section{TRACE: Transport Alignment Conformal Estimation}
\label{sec:trace}

The effectiveness of conformal prediction hinges on the choice
of the nonconformity score, which determines
how a scalar threshold maps back into a region in the output space.
We propose defining nonconformity scores directly through  
transport alignment.  They measure the degree to which candidate outputs deviate from the learned generative dynamics, 
while circumventing the need to
specify or estimate the likelihood $p(y \mid x)$ explicitly. 

\subsection{Transport-based nonconformity scores}

Although diffusion models and flow matching models differ in their formulation, both are trained by minimizing squared discrepancies of the form
\begin{equation}
\mathcal{L}(\theta)
=
\mathbb{E}_{(x,y),\, t,\, \xi}
\bigl[
\ell(x, y; t, \xi)
\bigr]\,.
\label{eq:consistency_loss}
\end{equation}
Here, the expectation, $\mathbb{E}$, is taken over the objects in the subscript, where $t$ is a uniformly distributed time index (over $\{1,\dots,T\}$ for diffusion, and over $[0,1]$ for flow matching) and $\xi$ is an auxiliary random variable (injected noise $\varepsilon$ for diffusion, or a reference sample $y_0$ for flow matching).
The instantaneous loss $\ell(\cdot)$ encodes the deviation between the model prediction and the ground-truth transport relation at $(t, \xi)$.

For a fixed input $x$, values of the output $y$ that are well aligned with the learned generative dynamics tend to incur smaller values of 
the loss, $\ell$, while misaligned outputs incur larger values of the loss.  
This motivates us to define the \emph{nonconformity score} of a candidate output $y$ at a given $x$ as the expected transport error
\begin{equation}
\bar{s}(x,y)
=
\mathbb{E}_{t,\xi}
\bigl[
\ell(x, y; t, \xi)
\bigr].
\label{eq:population_score}
\end{equation}
In practice, the expectation in \eqref{eq:population_score} is intractable but can be approximated via Monte Carlo sampling.
Let $\mathcal{T}$ denote a finite set containing the sampled time indices.  For example, $\mathcal{T}$ can be the evenly-spaced grid, $\{\frac{1}{T}, \frac{2}{T}, \cdots, 1\}$, for some user-specified $T$ 
 for diffusion, and a random sample or an evenly spaced grid of $[0,1]$ for flow matching. Let $\{\xi_{t, r}\}_{t\in\mathcal{T},\, r = 1, \dots, R}$ be independent samples of the auxiliary variable.
The \emph{empirical nonconformity score} is then given by
\begin{equation}
s(x,y)
=
\frac{1}{|\mathcal{T}| \cdot R}
\sum_{t \in \mathcal{T}}
\sum_{r=1}^{R}
\ell(x, y; t, \xi_{t,r}),
\label{eq:empirical_score}
\end{equation}
which serves as an 
estimator of $\bar{s}(x,y)$.

This construction does not rely on explicit density evaluation, invertibility, or Jacobian computation.
The score is fully determined by the model's predictive behavior along transport trajectories, making it applicable to generative models trained via denoising or velocity-matching objectives.

\subsection{TRACE-Diff: Diffusion-based score}
\label{sec:trace_diff}

In conditional diffusion models, the forward process perturbs a
target sample $y$ into a noisy latent state $y_t$ according to
\eqref{eq:forward_diffusion}.  
Given $(y_t,t,x)$, the model is trained to predict the injected
Gaussian noise $\varepsilon$, leading to the denoising loss
\[
\ell_{\mathrm{Diff}}(x,y;t,\varepsilon)
=
\left\|
\varepsilon-\hat{\varepsilon}_\theta(y_t,t,x)
\right\|^2 .
\]
This loss provides a natural measure of how well a candidate output $y$ aligns with 
the conditional data-generating distribution 
$p(y\mid x)$. After perturbation, a typical $y$ 
produces a noisy state $y_t$ in high-density regions 
of $p_t(y_t\mid x)$ where the model can accurately 
predict the injected noise, whereas an atypical $y$ 
maps to low-density regions where the denoising task 
is harder and the prediction error increases.

\paragraph{Connection to variational likelihood bounds.}
Besides the above intuitive argument, we provide an analytical 
justification for the diffusion-based score. 
\citet{Lei01032013} showed that, among all 
prediction regions at a given coverage level, density 
level sets achieve the minimum volume when the true 
conditional density is known. We now show that 
TRACE-Diff induces an ordering that approximates 
this oracle through the diffusion variational bound.

For a conditional DDPM with noise-prediction parameterization
$\hat{\varepsilon}_\theta$, the standard variational bound 
admits the decomposition
\citep{DBLP:journals/corr/abs-2006-11239, 
DBLP:journals/corr/abs-2107-00630}
\begin{equation}
-\log p_\theta(y \mid x)
\;\le\;
\mathcal{L}_{\mathrm{VLB}}(y \mid x)
\;=\;
\kappa(x)
\;+\;
\sum_{t=1}^{T}
\lambda_t \;
\mathbb{E}_{\varepsilon}\!\left[
\bigl\|
\varepsilon
-
\hat{\varepsilon}_\theta\bigl(
\sqrt{\bar{\alpha}_t}\, y 
+ \sqrt{1-\bar{\alpha}_t}\, \varepsilon,\;
t,\; x \bigr)
\bigr\|^2
\right],
\label{eq:vlb_decomposition}
\end{equation}
where 
$\lambda_t>0$ are schedule-dependent weights.
Since $\kappa(x)$ is constant in $y$, at a given $x$, the ordering of candidate values $y$
induced by the VLB-weighted score 
$\bar{s}_{\mathrm{VLB}}(x,y) := 
\sum_{t=1}^{T} \lambda_t\,
\mathbb{E}_\varepsilon[\|\varepsilon - 
\hat{\varepsilon}_\theta(\cdot)\|^2]$ 
is equivalent to the ordering induced by the 
variational bound itself.

\begin{proposition}[VLB ordering equivalence]
\label{prop:vlb_ordering}
For any fixed input $x$ and any two candidate outputs 
$y_1, y_2$,
$$
\bar{s}_{\mathrm{VLB}}(x, y_1)
\le
\bar{s}_{\mathrm{VLB}}(x, y_2)
\quad
\text{if and only if}\quad
\mathcal{L}_{\mathrm{VLB}}(y_1 \mid x)
\le
\mathcal{L}_{\mathrm{VLB}}(y_2 \mid x).
$$
\end{proposition}
\begin{proof}
By \eqref{eq:vlb_decomposition}, 
$\mathcal{L}_{\mathrm{VLB}}(y\mid x) = 
\kappa(x) + \bar{s}_{\mathrm{VLB}}(x,y)$, 
where $\kappa(x)$ does not depend on $y$. 
Hence the orderings coincide.
\end{proof}
\begin{remark}[Conformal regions as VLB level sets]
\label{cor:vlb_levelset}
Following proposition~\ref{prop:vlb_ordering}, the conformal prediction region using 
$\bar{s}_{\mathrm{VLB}}$ can be written as
$$
\widehat{C}(x)
=
\{y:\bar{s}_{\mathrm{VLB}}(x,y)\le q_{1-\alpha}\}
=
\{y:\mathcal{L}_{\mathrm{VLB}}(y\mid x)
\le \kappa(x)+q_{1-\alpha}\}.
$$
That is, the conformal region corresponds to a level 
set of the diffusion variational bound.
\end{remark}


\begin{remark}[Uniform vs.\ VLB weighting]
\label{rem:uniform_vs_vlb}
TRACE-Diff uses uniform weighting over time steps,
yielding the nonconformity score $\bar{s}(x,y)$
in \eqref{eq:population_score}.
The VLB decomposition instead uses schedule-dependent
weights $\lambda_t$, yielding $\bar{s}_{\mathrm{VLB}}(x,y)$.
The two weightings induce similar orderings when
$\lambda_t$ varies moderately over $t$.
Indeed, uniform weighting is standard in practice:
it down-weights loss terms at small $t$ relative to
the VLB weights, encouraging the model to focus on
more difficult denoising tasks at larger $t$, which
has been shown to yield better sample quality
\citep{DBLP:journals/corr/abs-2006-11239}.
\end{remark}

\subsection{TRACE-FM: Flow-matching-based score}
\label{sec:trace_fm}

For flow matching models, the transport path connects a
reference sample $y_0 \sim \mathcal{N}(0,I)$ to the data
sample $y$ via linear interpolation as in
\eqref{eq:fm_path}.
Given the interpolated state $(y_t, t, x)$, the model is
trained to predict the transport velocity $y - y_0$.
This yields the instantaneous loss
\[
\ell_{\mathrm{FM}}(x, y; t, y_0)
=
\bigl\|
\hat{v}_\theta\bigl(
(1-t)\, y_0 + t\, y,\;
t,\; x\bigr)
-
(y - y_0)
\bigr\|^2.
\]
If $y$ is consistent with $p(y\mid x)$, then the 
interpolated state $y_t$ lies in regions where the model 
has accurately learned the velocity field, and the 
prediction error is small; whereas if $y$ is incompatible with 
$p(y\mid x)$, the error increases.

Compared with the diffusion-based score, flow matching
benefits from optimal-transport paths that encourage
straight trajectories, so the velocity field varies
more smoothly across reference samples $y_0$.
This leads to lower Monte Carlo variance in the
empirical score, as observed empirically in
Section~\ref{sec:ablation_mc}.

\subsection{Deterministic evaluation via common random numbers}
\label{sec:crn}

The empirical scores in \eqref{eq:empirical_score} 
involve Monte Carlo averaging over auxiliary random variables
(noise $\varepsilon$ for diffusion, reference samples $y_0$ for flow matching).
If independent random draws were used for each score evaluation,
the resulting nonconformity scores would themselves be random,
introducing extra variability into the calibration ordering.

To eliminate this source of randomness, we adopt a
\emph{common random numbers} (CRN) strategy.
Prior to any score computation, we draw a single bank of auxiliary variables
$$
\omega
=
\{\xi_{t,r}\}_{t \in \mathcal{T},\, r=1,\dots,R},
$$
where $\xi_{t,r}=\varepsilon_{t,r}$ in diffusion models and
$\xi_{t,r}=y_0^{t,r}$ in flow matching.
This bank is generated once and reused for all score evaluations.

Conditioned on a fixed bank $\omega$,
the resulting score
$$
s_\omega(x,y)
=
\frac{1}{|\mathcal{T}| \cdot R}
\sum_{t \in \mathcal{T}}
\sum_{r=1}^{R}
\ell(x,y;t,\xi_{t,r})
$$
is a deterministic function of $(x,y)$,
and 
the nonconformity ordering
across calibration and test points is consistently
determined. 
Hence, all exchangeability-based
coverage guarantees hold for conformal prediction, independent of how accurately
$s_\omega$ approximates the ideal population score.

Despite the theoretical guarantee, practitioners may view it as a potential limitation of the CRN method that different draws of the CRN bank can produce different scoring functions, leading to varying calibration thresholds and prediction regions. 
Section~\ref{sec:theory} showed that this bank-level variability decreases at a controlled rate with the total Monte Carlo budget $|\mathcal{T}| \cdot R$, and empirical experiments suggest that, with a moderately sized bank, this variability has little practical impact.

\begin{algorithm}[ht]
\caption{TRACE: Transport Alignment Conformal Estimation}
\label{alg:consistency_conformal}
\begin{algorithmic}[1]
\State \textbf{Input:}
\begin{itemize}
    \item Dataset $\mathcal{D} = \{(x_i, y_i)\}_{i=1}^n$
    \item Miscoverage level $\alpha \in (0,1)$
    \item Generative model class (diffusion or flow matching)
    \item Monte Carlo parameters: time indices $\mathcal{T}$, the number of repeats $R$
    \item Test input $x_{\mathrm{new}}$
\end{itemize}

\State \textbf{Training phase:}
\begin{enumerate}
    \item Randomly split $\mathcal{D}$ into training set $\mathcal{D}_{\mathrm{train}}$ and calibration set $\mathcal{D}_{\mathrm{cal}} = \{(x_i, y_i)\}_{i=1}^{n_{\mathrm{cal}}}$.
    \item Train a conditional generative model $\hat{\varepsilon}_\theta$ (diffusion) or $\hat{v}_\theta$ (flow matching) on $\mathcal{D}_{\mathrm{train}}$.
\end{enumerate}

\State \textbf{CRN construction:}
\begin{enumerate}
    \setcounter{enumi}{2}
    \item Pre-generate a fixed bank of auxiliary variables shared across all evaluations:
    \begin{itemize}
        \item Diffusion: $\{\varepsilon_{t,r}\}_{t \in \mathcal{T},\, r=1,\dots,R}$, \; $\varepsilon_{t,r} \sim \mathcal{N}(0, I)$
        \item Flow matching: $\{y_0^{t,r}\}_{t \in \mathcal{T},\, r=1,\dots,R}$, \; $y_0^{t,r} \sim \mathcal{N}(0, I)$
    \end{itemize}
\end{enumerate}

\State \textbf{Calibration phase:}
\begin{enumerate}
    \setcounter{enumi}{3}
    \item For each $(x_i, y_i) \in \mathcal{D}_{\mathrm{cal}}$, compute the nonconformity score $s(x_i, y_i)$ via \eqref{eq:empirical_score} using the shared CRN bank.
    \item Compute the conformal threshold $q_{1-\alpha}$.
\end{enumerate}

\State \textbf{Output:}
Prediction region for $x_{\mathrm{new}}$:
$$\widehat{C}(x_{\mathrm{new}}) = \bigl\{\, y \in \mathcal{Y} : s(x_{\mathrm{new}}, y) \le q_{1-\alpha} \,\bigr\}.$$
\end{algorithmic}
\end{algorithm}

\section{Theoretical Analysis of Finite-Budget TRACE Scores}
\label{sec:theory}

The preceding section showed that TRACE scores of \eqref{eq:population_score} induce a useful ordering that essentially respects the diffusion variational bound. In practice, empirical scores are used, 
relying on
Monte Carlo approximation via a sample of time points, $\mathcal{T}$, and the corresponding auxiliary bank $\omega$. 
Further, for FM, the TRACE score is an integral over continuous time $[0,1]$, which is in practice approximated by sum over a discrete set of times. 

We quantify bank-induced approximation errors in Sections~\ref{sec:score_approx}–\ref{sec:region_implications} and time-discretization errors in Section~\ref{sec:discretization}; the former affects all TRACE methods, while the latter applies only to TRACE-FM. 
We discuss the impact of these errors on the conformal calibration threshold
and volume of the prediction region.
Throughout this section, $B = |\mathcal{T}| \cdot R$
denote the total Monte Carlo budget.

\subsection{Score Approximation Error}\label{sec:score_approx}
The approximation error of the empirical score $s_\omega$
for the continuous-time population score $\bar{s}$ has two components:
the bank-induced stochastic error, 
$s_\omega - \bar{s}_{\mathcal{T}}$; 
and the time-discretization error,  
$\bar{s}_{\mathcal{T}}-\bar{s}$.
Sections~\ref{sec:score_approx}--\ref{sec:region_implications} analyze
the first component;
Section~\ref{sec:discretization} bounds the second.

Fix a 
set of time indices $\mathcal{T}$.
The time-discretized population score is
\begin{equation}
\bar{s}_{\mathcal{T}}(x,y)
=
\frac{1}{|\mathcal{T}|}
\sum_{t \in \mathcal{T}}
\mathbb{E}_{\xi}
\bigl[
\ell(x,y;t,\xi)
\bigr].
\label{eq:population_score_T_journal}
\end{equation}

The empirical score constructed from a CRN bank
$\omega = \{\xi_{t,r}\}_{t \in \mathcal{T}, r=1,\dots,R}$
is
\begin{equation}
s_\omega(x,y)
=
\frac{1}{|\mathcal{T}|R}
\sum_{t \in \mathcal{T}}
\sum_{r=1}^{R}
\ell(x,y;t,\xi_{t,r}).
\label{eq:empirical_score_journal}
\end{equation}


\begin{assumption}[Finite second moment and independent auxiliary draws]
\label{assump:moment}
For all $(x,y)$ and $t \in \mathcal{T}$,
$\mathbb{E}_{\xi}[\ell(x,y;t,\xi)^2] < \infty$.
The auxiliary variables $\{\xi_{t,r}\}_{t \in \mathcal{T},\, r=1,\dots,R}$
are mutually independent.
\end{assumption}
\begin{theorem}[Score approximation rate]
\label{thm:score_mse}
Under Assumption~\ref{assump:moment},
for any fixed $(x,y)$,
$$
\mathbb{E}_\omega\!\left[
\bigl(
s_\omega(x,y)
-
\bar{s}_{\mathcal{T}}(x,y)
\bigr)^2
\right]
=
\frac{1}{|\mathcal{T}|^2}
\sum_{t \in \mathcal{T}}
\frac{
\mathrm{Var}_{\xi}
\bigl(
\ell(x,y;t,\xi)
\bigr)
}{R}.
$$
Consequently,
$$
\mathbb{E}_\omega
\bigl[\bigl|
s_\omega(x,y)
-
\bar{s}_{\mathcal{T}}(x,y)
\bigr|\bigr]
=
O\!\left(\frac{1}{\sqrt{B}}\right).
$$
\end{theorem}
The proof is provided in Appendix~\ref{app:proof_score_mse}. This basic result shows that the bank-level mean error decays at rate $O(1/\sqrt{B})$ and underpins the subsequent analysis of conformal regions.

\subsection{Threshold Stability Under Bank Perturbations}\label{sec:threshold_stability}

Recall $\mathcal{D}_{\mathrm{cal}} =\{(x_i,y_i)\}_{i=1}^{n_{\mathrm{cal}}}$ denotes the calibration set, with corresponding calibration scores
$s_\omega(x_i,y_i)$ and conformal threshold
$$
q_{1-\alpha}^{\omega}
=
\text{the }
\left\lceil (1-\alpha)(n_{\mathrm{cal}}+1) \right\rceil
\text{-th order statistic}.
$$
Let $q_{1-\alpha}^{\mathcal{T}}$
denote the analogous threshold constructed using the ideal
scores $\bar{s}_{\mathcal{T}}(x_i,y_i)$ over $\mathcal{D}_{\mathrm{cal}}$.

\begin{theorem}[Threshold stability]
\label{thm:threshold_rate}
Under Assumption~\ref{assump:moment},
$$
\mathbb{E}_\omega
\bigl[
\bigl|
q_{1-\alpha}^{\omega}
-
q_{1-\alpha}^{\mathcal{T}}
\bigr|
\bigr]
\le
2\sqrt{\frac{n_{\mathrm{cal}}\, C}{B}},
$$
where $C := \max_{1 \le i \le n_{\mathrm{cal}}} \sup_{t \in \mathcal{T}} \mathrm{Var}_\xi\bigl(\ell(x_i,y_i;t,\xi)\bigr)$.
\end{theorem} 
This theorem bounds the error of using CRN-based conformal threshold over that of the ideal scores.    The proof can be found in Appendix~\ref{app:proof_threshold}, which
combines Theorem~\ref{thm:score_mse}
with deterministic stability properties of order statistics. 

\subsection{Implications for Prediction Regions}\label{sec:region_implications}

For each fixed bank $\omega$,
the conformal region is
$$
\widehat{C}^{\omega}(x)
=
\{
y :
s_\omega(x,y)
\le
q_{1-\alpha}^{\omega}
\}.
$$
The volume of this sublevel set depends on the threshold $q_{1-\alpha}^{\omega}$,
which in turn depends on the bank $\omega$. When the function $s$
is continuous in $y$ and the level-set boundary
$\{y : s_\omega(x,y) = q\}$ has finite surface measure,  $V_\omega(x,q)$ is Lipschitz in $q$, which is typical for smooth neural-network-based TRACE scores. 
Under these conditions,  
different banks yield
similar region volumes, and  
the bank-induced error in volume
vanishes at rate $O(1/\sqrt{B})$: 
\begin{corollary}[Volume stability]
\label{cor:volume_stability}
For a given bank $\omega$, let 
$$
V_\omega(x,q)
=
\mathrm{Vol}\bigl(\{y : s_\omega(x,y) \le q\}\bigr)
$$
denote the prediction region volume as a function of the threshold. 
Assume that $V_\omega(x,q)$ is locally Lipschitz in $q$ with constant $L$
uniformly over $\omega$. 
Then under Assumption~\ref{assump:moment},
$$
\mathbb{E}_\omega
\bigl[
|V_\omega(x,q_{1-\alpha}^{\omega})
-
V_\omega(x,q_{1-\alpha}^{\mathcal{T}})|
\bigr]
\le
L \,
\mathbb{E}_\omega
\bigl[
|q_{1-\alpha}^{\omega}
-
q_{1-\alpha}^{\mathcal{T}}|
\bigr]
=
O\!\left(\frac{1}{\sqrt{B}}\right).
$$
\end{corollary}


\begin{remark}[High-probability bound]
\label{rem:high_prob}
By Markov's inequality, Corollary~\ref{cor:volume_stability}
implies that for any $\delta > 0$,
$$
\mathbb{P}_\omega\!\left(
|V_\omega(x, q_{1-\alpha}^{\omega}) - V_\omega(x, q_{1-\alpha}^{\mathcal{T}})|
> \delta
\right)
\leq
\frac{2L\sqrt{n_{\mathrm{cal}}\, C / B}}{\delta}.
$$
\end{remark}

\begin{remark}[Marginal coverage guarantee] 
The above results concern the variation in volume incurred by the prediction region due to finite-budget approximation. Approximation error affects efficiency but not validity. That is, for any fixed $\omega$,
the scoring rule $s_\omega$ is deterministic
and applied consistently to calibration and test points.
Therefore, under exchangeability,
$$
\mathbb{P}\bigl(
Y \in \widehat{C}^{\omega}(X)
\bigr)
\ge 1-\alpha
$$
holds exactly for every finite budget $B$.
\end{remark}

\subsection{Time-Discretization Error}
\label{sec:discretization}

The preceding analysis fixes a discrete time set 
$\mathcal{T}$ and quantifies the bank-induced error 
$s_\omega - \bar{s}_{\mathcal{T}}$. For flow matching, 
the time index $t$ is continuous on $[0,1]$, so the 
 score involves an integral over time,
$$
\bar{s}(x,y)
=
\int_0^1
\mathbb{E}_{\xi}[\ell(x,y;t,\xi)]\, dt,
$$
and replacing it with a finite sum over 
$\mathcal{T}$ introduces a discretization error. 
For diffusion models with a discrete time schedule 
$t \in \{1,\dots,T\}$, this error is zero when 
$\mathcal{T}$ contains all diffusion steps. We now bound the discretization error for the 
continuous-time case.



\begin{assumption}[Lipschitz regularity in time]
\label{assump:lipschitz_time}
For all $(x,y)$, the function
$t \mapsto \mu(t;x,y) := \mathbb{E}_{\xi}[\ell(x,y;t,\xi)]$
is $L$-Lipschitz on $[0,1]$:
$$
\bigl| \mu(t_1;x,y) - \mu(t_2;x,y) \bigr|
\le
L\, |t_1 - t_2|
\qquad
\forall\; t_1, t_2 \in [0,1].
$$
\end{assumption}

\begin{proposition}[Discretization error bound]
\label{prop:discretization}
Under Assumption~\ref{assump:lipschitz_time},
if $\mathcal{T} = \{t_1, \dots, t_m\}$ is a uniform grid on $[0,1]$
with spacing $\Delta = 1/m$,
then for any $(x,y)$,
$$
\bigl|
\bar{s}(x,y)
-
\bar{s}_{\mathcal{T}}(x,y)
\bigr|
\le
\frac{L}{2\,|\mathcal{T}|}.
$$
\end{proposition}

\begin{proof}
Write
$\bar{s}(x,y) = \int_0^1 \mu(t)\, dt$
and
$\bar{s}_{\mathcal{T}}(x,y) = \frac{1}{m}\sum_{j=1}^m \mu(t_j)$.
Under the Lipschitz condition, the standard rectangle-rule
quadrature error satisfies
$|\int_0^1 \mu\, dt - \frac{1}{m}\sum_j \mu(t_j)|
\le L \Delta / 2 = L / (2m)$. Details can be found in Appendix \ref{app:proof_discretization}.
\end{proof}


\begin{remark}[Combined error budget]
\label{rem:combined_error}
Combining Theorem~\ref{thm:score_mse}
and Proposition~\ref{prop:discretization}
via the triangle inequality gives
$$
\mathbb{E}_\omega\!\left[
\bigl|
s_\omega(x,y) - \bar{s}(x,y)
\bigr|
\right]
\le
\underbrace{
\frac{L}{2\,|\mathcal{T}|}
}_{\text{discretization}}
+
\underbrace{
O\!\left(\frac{1}{\sqrt{B}}\right)
}_{\text{bank-induced}}.
$$
where the first term is the deterministic discretization error and
the second arises from Monte Carlo approximation with total budget
$B = |\mathcal{T}|\cdot R$.
For a fixed budget $B$, increasing the number of time steps
$|\mathcal{T}|$ reduces discretization error while increases
the bank-induced variance,  
and vice versa.
This bias--variance trade-off is consistent with the empirical
ablation results in Section~\ref{sec:ablation_mc}.
\end{remark}

\section{Experiments}
\label{sec:experiments}

\subsection{Experimental setup}

\paragraph{Datasets.}
We experiment with two synthetic and five real-world datasets covering a range of output geometries and dimensionalities.

\noindent \textbf{(1) Synthetic data.} The synthetic datasets are generated under the general form, $Y = k \cdot f(X) + \varepsilon$, where $Y\in \mathbb{R}^2$, $X\in \mathbb{R}^p$, 
$f$ is chosen from various nonlinear functions, $k>0$ is a scalar controlling the signal strength, and $\varepsilon$ is independent noise with  non-Gaussian distributions that have non-convex contours or multiple modes. Specifically, we use a $2 \times 2$ factorial design in our experiment
(Table~\ref{tab:2x2_design}) that varies two factors:
\emph{noise type}
(unimodal spiral vs.\ six-component pinwheel) and
\emph{model difficulty}
($X \in \mathbb{R}^2$ with $k{=}1$ vs.\
$X \in \mathbb{R}^7$ with $k{=}5$ and five nuisance
dimensions), with details provided in Appendix~\ref{app:synthetic_data}. Briefly, for the noise type, the distribution in \textsc{Spiral} has curved, non-convex contours, while that in \textsc{Pinwheel}  consists of six elongated Gaussian clusters, forming deep concavities between the arms. Examples of the resulting geometry of the conditional density $y | X=x$ can be seen in Figure~\ref{fig:regions_synthetic}. For model difficulty, regardless of the dimension of the input, only two of them affect the output: the 7-dimensional case included five nuisance features to test the model's ability to identify the informative ones. 
Each of the four datasets has $n = 30{,}000$ samples, with $y$
normalized to zero mean and unit variance.

\begin{table}[h]
\small
\centering
\caption{$2 \times 2$ factorial design for synthetic
experiments, varying noise types (columns) and model complexity (rows).}
\label{tab:2x2_design}
\begin{tabular}{lcc}
\toprule
 & Spiral noise (unimodal) & Pinwheel noise (6-component) \\
\midrule
$X \in \mathbb{R}^2,\; k{=}1$
  & \textsc{Spiral$_L$} & \textsc{Pinwheel$_L$} \\
$X \in \mathbb{R}^7,\; k{=}5$
  & \textsc{Spiral$_H$} & \textsc{Pinwheel$_H$} \\
\bottomrule
\end{tabular}
\end{table}

\noindent \textbf{(2) Real-world data.}
\textsc{Taxi} contains New York City taxi trip records; the task is to predict drop-off coordinates from pick-up location and time features.
\textsc{Energy} \citep{TSANAS2012560} predicts heating and cooling loads for residential buildings from eight structural attributes.
\textsc{RF1-2D} and \textsc{RF1-4D} \citep{Spyromitros_Xioufis_2016} forecast river flows 48 hours ahead at 2 and 4 sites in the Mississippi River network, giving $Y \in \mathbb{R}^2$ and $Y \in \mathbb{R}^4$ respectively.
\textsc{SCM20D} \citep{Spyromitros_Xioufis_2016} predicts next-day mean prices in a supply chain management scenario.
All real-world datasets are normalized internally, and reported volumes are rescaled to the original $Y$-space.
Table~\ref{tab:datasets} summarizes the dataset characteristics.

\begin{table}[ht]
\centering
\caption{Dataset summary.}
\label{tab:datasets}
\small
\begin{tabular}{llrrr}
\toprule
Dataset & Type & $n$ & $\dim(X)$ & $\dim(Y)$ \\
\midrule
\textsc{Spiral$_L$}   & Synthetic & 30{,}000 & 2  & 2 \\
\textsc{Spiral$_H$}   & Synthetic & 30{,}000 & 7  & 2 \\
\textsc{Pinwheel$_L$} & Synthetic & 30{,}000 & 2  & 2 \\
\textsc{Pinwheel$_H$} & Synthetic & 30{,}000 & 7  & 2 \\
\midrule
\textsc{Taxi}     & Real & 6{,}000  & 2  & 2 \\
\textsc{Energy}   & Real &   768  & 8  & 2 \\
\textsc{RF1-2D}   & Real & 9{,}005  & 20 & 2 \\
\textsc{RF1-4D}   & Real & 9{,}005  & 20 & 4 \\
\textsc{SCM20D}   & Real & 5{,}000  & 61 & 2 \\
\bottomrule
\end{tabular}
\end{table}
\paragraph{Methods.}
We compare the proposed TRACE (TRACE-Diff and TRACE-FM) against latent-density-based methods CONTRA and JAPAN, a sample-based baseline (PCP-Diff), and shape-restricted baselines RCP \citep{johnstone2021conformal}, 
NLE \citep{messoudi2022ellipsoidal}, and MCQR \citep{fang2025contra}.
All methods have been described in Sections~\ref{sec:related} and~\ref{sec:trace}.

\paragraph{Evaluation.}
For each method, we report the empirical coverage on the test set and the average volume of the prediction regions.
The volume of multi-dimensional regions are estimated via quasi-Monte Carlo integration using Sobol sequences over local bounding boxes.
All experiments use miscoverage level $\alpha = 0.1$ and are repeated $20$ times with different random seeds; results are reported as mean $\pm$ standard deviation.

\paragraph{Implementation.}
DDPM and FM models typically use a hidden dimension of $256$,
a conditioning dimension of $128$, and $8$--$10$
FiLM-conditioned residual blocks, trained with exponential
moving average updates (decay $0.999$).
A conditional normalizing flow is trained separately to
support the baseline methods (CONTRA, JAPAN, RCP, NLE, MCQR);
most datasets use RealNVP coupling layers, with neural
spline flows (NSF) substituted where additional expressiveness
is needed.
PCP-Diff uses samples from the diffusion model.
For \textsc{Pinwheel}, the number of coupling layers and
the learning rate of the normalizing flow are increased
to accommodate its more complex conditional structure.
TRACE scores are computed with a budget of
$|\mathcal{T}| = 15$ time steps and $R = 8$ repeats
($B = 120$) for most datasets.
A shared CRN bank is pre-generated and reused across all
score evaluations as described in Section~\ref{sec:crn}.
The dataset is randomly split into training ($67.5\%$),
calibration ($22.5\%$), and test ($10\%$) subsets.
\subsection{Results}
\label{sec:results}
We now present the main experimental findings.
Tables~\ref{tab:results_synthetic},~\ref{tab:taxi_results}
and~\ref{tab:results_real} report empirical coverage and
prediction region volume across all methods and datasets. All methods attain empirical coverage close to the nominal $90\%$ level, consistent with the conformal guarantee.
We therefore focus the discussion on prediction region volume, which reflects the informativeness of the uncertainty sets, and use region visualizations (Figures \ref{fig:regions_synthetic}) to explain the observed differences.

Across all nine datasets, TRACE-FM achieves
the smallest volume in seven cases (bold entries in
Tables~\ref{tab:results_synthetic},~\ref{tab:taxi_results}
and~\ref{tab:results_real}) and is statistically
competitive on the remaining two. TRACE-Diff consistently tracks TRACE-FM with slightly larger volumes. The latent-density methods JAPAN perform best
on the two lower-dimensional synthetic settings but degrade
substantially as model complexity increases.
The following subsections analyze these patterns in
detail.


\subsubsection{Synthetic datasets}
\label{sec:results_synthetic}

\begin{figure}[ht]
\centering
\begin{minipage}[b]{\textwidth}
    \centering
    \includegraphics[width=\textwidth]{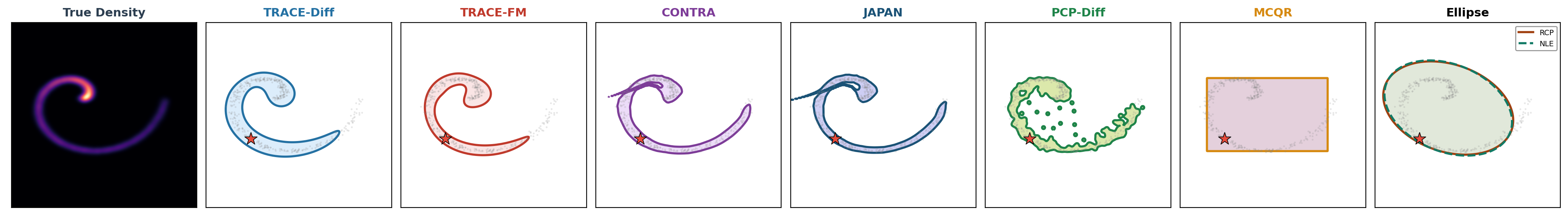}
    \centerline{\small (a) \textsc{Spiral$_L$}}
\end{minipage}
\begin{minipage}[b]{\textwidth}
    \centering
    \includegraphics[width=\textwidth]{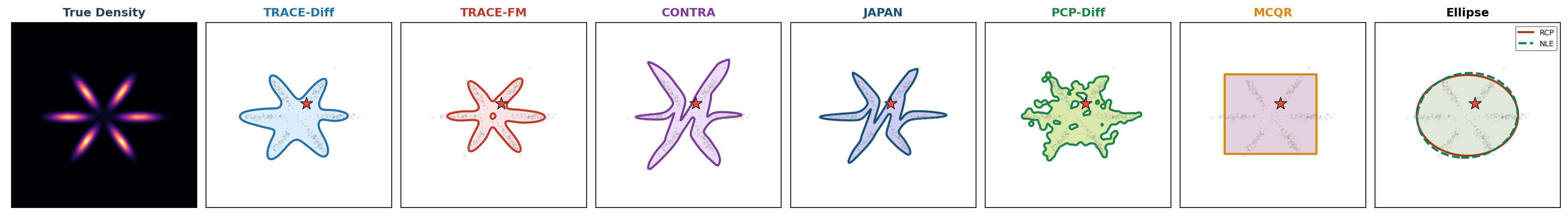}
    \centerline{\small (b) \textsc{Pinwheel$_L$}}
\end{minipage}
\begin{minipage}[c]{\textwidth}
    \centering
    \includegraphics[width=\textwidth]{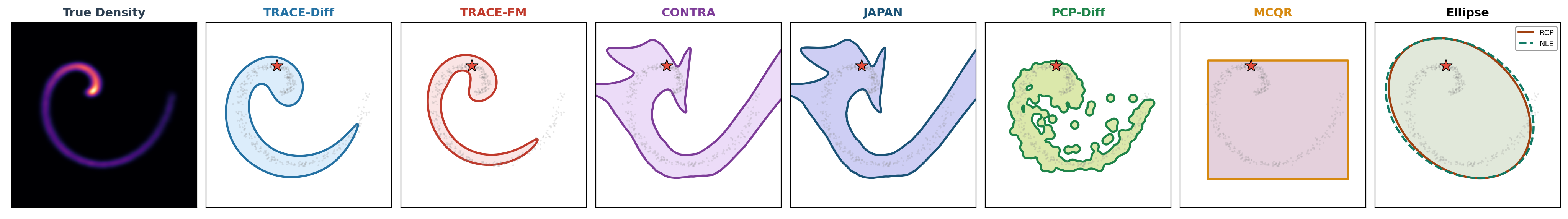}
    \centerline{\small (c) \textsc{Spiral$_H$}}
\end{minipage}
\begin{minipage}[d]{\textwidth}
    \centering
    \includegraphics[width=\textwidth]{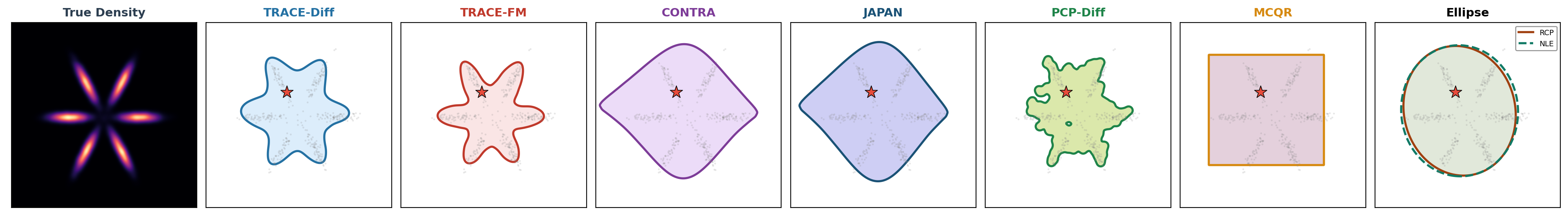}
    \centerline{\small (d) \textsc{Pinwheel$_H$}}
\end{minipage}

\caption{Prediction regions on synthetic datasets for a representative test input.
The leftmost panel shows the true conditional density.
Subsequent panels show the $90\%$ conformal prediction region for each method.
Gray dots are calibration samples; the red star marks the actual output $y$.}
\label{fig:regions_synthetic}
\end{figure}

\paragraph{Impact of Noise Type on Performance}

Comparing spiral and pinwheel columns within each
conditioning regime reveals that the latent-density
methods are more sensitive to noise complexity than
TRACE.
In the lower-dimensional regime, moving from
\textsc{Spiral$_L$} to \textsc{Pinwheel$_L$} increases
JAPAN's volume by a factor of $2.3$ and CONTRA's by
$2.6$, whereas TRACE-FM increases by only $1.4$.
In the higher-dimensional regime, the latent-density
methods have already saturated at large volumes, so
the additional effect of noise complexity is smaller
in relative terms; TRACE-FM increases by a factor of
$1.7$ from \textsc{Spiral$_H$} to \textsc{Pinwheel$_H$}.
The underlying reason is that accurate density estimation
of six well-separated modes with deep concavities is
inherently difficult for normalizing flows; CONTRA is
further disadvantaged by its connectivity constraint,
which forces the predicted region to be connected and
thus to fill the gaps between arms.
TRACE is less affected because it evaluates candidate
quality pointwise through reconstruction error rather
than through global density estimation.
Shape-restricted baselines are largely insensitive to
noise geometry, but their volumes remain consistently
the largest due to rigid geometric templates.
Figure~\ref{fig:regions_synthetic} illustrates these
differences: in rows (a) and (b), all generative methods
trace the target geometry in the lower-dimensional regime,
with JAPAN and CONTRA producing slightly tighter
boundaries on spiral but expanding visibly on pinwheel;
in rows (c) and (d), the contrast is more pronounced,
as CONTRA and JAPAN inflate into convex shapes while
TRACE continues to resolve the star-shaped structure.

\paragraph{Impact of Model Difficulty on Performance}
Increasing the input dimension from $\mathbb{R}^2$
 to $\mathbb{R}^7$ with five nuisance features is even more damaging to the latent-density methods. The normalizing flow must simultaneously learn a higher-dimensional conditioning mapping and an expressive invertible transform; the added burden erodes its precision on fine-grained noise geometry, and the prediction regions inflate dramatically. TRACE is far more robust to this shift, as its volumes remain comparable across the two input regimes. In Figure~\ref{fig:regions_synthetic}, comparing rows (a) and (c), the CONTRA and JAPAN regions expand well beyond the true support, while TRACE regions remain comparably tight; comparing rows (b) and (d), CONTRA and JAPAN lose the star-shaped geometry entirely and collapse into diamond shapes, while TRACE preserves its adaptive geometry.
\begin{table}[ht]
\centering
\begin{minipage}{0.99\linewidth}
\captionsetup{width=\linewidth}
\caption{Coverage (Cov, \%) and volume (Vol) on synthetic datasets (mean $\pm$ std over 20 repeats). Best volume in \textbf{bold}.}
\label{tab:results_synthetic}
\small
\setlength{\tabcolsep}{2.5pt}
\renewcommand{\arraystretch}{1.08}
\begin{tabular}{l cc cc cc cc}
\toprule
& \multicolumn{2}{c}{\textsc{Spiral$_L$}} & \multicolumn{2}{c}{\textsc{Pinwheel$_L$}} & \multicolumn{2}{c}{\textsc{Spiral$_H$}} & \multicolumn{2}{c}{\textsc{Pinwheel$_H$}} \\
\cmidrule(lr){2-3} \cmidrule(lr){4-5} \cmidrule(lr){6-7} \cmidrule(lr){8-9}
Method & Cov(\%) & Vol & Cov(\%) & Vol & Cov(\%) & Vol & Cov(\%) & Vol \\
\midrule
\multicolumn{9}{l}{\emph{Transport-based (proposed)}}\\
\quad TRACE-Diff     & 90.0 {\tiny$\pm$ 0.6} & 23.7 {\tiny$\pm$ 0.5} & 90.2 {\tiny$\pm$ 0.9} & 32.2 {\tiny$\pm$ 2.0} & 89.9 {\tiny$\pm$ 0.8} & 30.3 {\tiny$\pm$ 1.4} & 90.1 {\tiny$\pm$ 0.7} & 46.5 {\tiny$\pm$ 1.4} \\
\quad TRACE-FM       & 90.1 {\tiny$\pm$ 0.7} & 20.6 {\tiny$\pm$ 0.7} & 90.2 {\tiny$\pm$ 1.0} & 29.2 {\tiny$\pm$ 1.0} & 90.0 {\tiny$\pm$ 0.6} & \textbf{25.1} {\tiny$\pm$ 1.5} & 90.2 {\tiny$\pm$ 0.8} & \textbf{42.1} {\tiny$\pm$ 1.2} \\
\addlinespace[2pt]
\multicolumn{9}{l}{\emph{Latent-density-based}}\\
\quad CONTRA   & 90.1 {\tiny$\pm$ 0.8} & 14.2 {\tiny$\pm$ 0.7} & 90.0 {\tiny$\pm$ 1.2} & 37.5 {\tiny$\pm$ 4.7} & 90.0 {\tiny$\pm$ 0.6} & 64.9 {\tiny$\pm$ 6.9} & 90.2 {\tiny$\pm$ 0.6} & 85.2 {\tiny$\pm$ 3.1} \\
\quad JAPAN    & 90.2 {\tiny$\pm$ 0.5} & \textbf{12.5} {\tiny$\pm$ 0.5} & 90.4 {\tiny$\pm$ 1.0} & \textbf{28.9} {\tiny$\pm$ 2.6} & 90.0 {\tiny$\pm$ 0.6} & 64.8 {\tiny$\pm$ 6.8} & 90.2 {\tiny$\pm$ 0.6} & 82.3 {\tiny$\pm$ 2.7} \\
\addlinespace[2pt]
\multicolumn{9}{l}{\emph{Sample-based}}\\
\quad PCP-Diff & 90.1 {\tiny$\pm$ 0.5} & 18.9 {\tiny$\pm$ 0.3} & 90.1 {\tiny$\pm$ 0.9} & 37.4 {\tiny$\pm$ 1.5} & 89.8 {\tiny$\pm$ 0.6} & 34.9 {\tiny$\pm$ 1.2} & 90.1 {\tiny$\pm$ 0.6} & 49.2 {\tiny$\pm$ 1.0} \\
\addlinespace[2pt]
\multicolumn{9}{l}{\emph{Shape-restricted}}\\
\quad RCP      & 90.0 {\tiny$\pm$ 0.7} & 63.4 {\tiny$\pm$ 1.5} & 90.2 {\tiny$\pm$ 1.1} & 58.7 {\tiny$\pm$ 0.9} & 89.8 {\tiny$\pm$ 0.6} & 77.7 {\tiny$\pm$ 2.3} & 90.0 {\tiny$\pm$ 0.7} & 75.6 {\tiny$\pm$ 1.3} \\
\quad NLE      & 90.1 {\tiny$\pm$ 0.7} & 63.5 {\tiny$\pm$ 1.5} & 90.2 {\tiny$\pm$ 1.1} & 59.0 {\tiny$\pm$ 1.1} & 89.9 {\tiny$\pm$ 0.5} & 78.5 {\tiny$\pm$ 2.7} & 90.1 {\tiny$\pm$ 0.7} & 77.1 {\tiny$\pm$ 1.5} \\
\quad MCQR     & 90.1 {\tiny$\pm$ 0.7} & 61.5 {\tiny$\pm$ 0.3} & 89.7 {\tiny$\pm$ 0.8} & 63.3 {\tiny$\pm$ 0.9} & 89.7 {\tiny$\pm$ 0.5} & 77.7 {\tiny$\pm$ 1.6} & 90.0 {\tiny$\pm$ 0.6} & 84.8 {\tiny$\pm$ 1.4} \\
\bottomrule
\end{tabular}
\end{minipage}
\end{table}
\paragraph{Cross-over Impact.} On the easiest setting (Spiral$_L$), the latent-density methods produce the tightest regions. On the hardest setting (Pinwheel$_H$), the ranking fully reverses: TRACE achieves roughly half the volume of the latent-density methods. The two difficulty factors interact more severely for latent-density methods: from Spiral$_L$ to Pinwheel$_H$, JAPAN's volume increases by a factor of 6.6, driven primarily by a sharp sensitivity to model difficulty, whereas TRACE-FM increases by only a factor of 2.0.

\subsubsection{Real datasets}
\label{sec:results_real}
We now turn to the real-world datasets. Since the true conditional density is unavailable, we evaluate region quality through empirical coverage and volume. For \textsc{Taxi}, we additionally compare against a kernel density estimate as a visual reference.

\begin{table}[b]
\centering
\caption{\small \textsc{Taxi}: Coverage and volume for $90\%$ conformal regions.
Each entry is the mean over 20 random splits (standard error in parentheses).
For volume, the average volume and standard error is reported in units of $\times 10^{-2}$.
The smallest volume is in \textbf{bold}.}
\label{tab:taxi_results}
\small
\setlength{\tabcolsep}{2pt}
\renewcommand{\arraystretch}{1.05}
\begin{tabular}{lcccccccc}
\toprule
Metric & TRACE-Diff & TRACE-FM & CONTRA & JAPAN & PCP-Diff & RCP & NLE & MCQR \\
\midrule
Cov(\%) & 89.8(1.6) & 89.6(1.8) & 90.2(1.0) & 90.2(1.8) & 90.0(1.7) & 89.9(1.5) & 90.0(1.5) & 89.8(1.4) \\
Vol($\times 10^{-2}$) & \textbf{0.4}(0.0) & 0.5(0.0) & 1.0(0.1) & 0.5(0.1) & 0.5(0.0) & 0.6(0.0) & 0.6(0.1) & 0.6(0.0) \\
\bottomrule
\end{tabular}
\end{table}

\paragraph{NYC Taxi Data}

Figure~\ref{fig:regions_taxi} shows the 90\% prediction regions of the drop off location $y$ from various methods at a given taxi pick up location $x$. The red region in Panel (a) represents the highest-density area from a naive KDE estimate of the conditional distribution, which provides no coverage guarantee. In particular, the actual coverage can fall below the nominal level if the KDE estimate at $x$ is poor, for example due to limited sample size or strong heterogeneity in its neighborhood. This is why we need conformal prediction regions in (b)--(h) that provide coverage guarantees. Indeed, Table~\ref{tab:taxi_results} shows that all conformal regions achieved  coverage close to the nominal level and had volumes (areas in longitude--latitude coordinates) in the range of $0.004$--$0.010$. TRACE-Diff produces a compact primary region closely aligned with the  
high-density area of the KDE, including capturing a small disconnected component.
TRACE-FM attains a comparable volume and yields a single connected region. Among the latent-density approaches, CONTRA generates a 
smooth region covering a wider area, while JAPAN exhibits a more irregular and dispersed shape. The sample-based PCP-Diff construction appears fragmented into many scattered components. Finally, the shape-restricted methods, MCQR, RCP, and NLE, produce smooth but unnecessarily large regions that include low-probability drop-off locations. 
Overall, although the quantitative volumes are similar, TRACE-Diff provides the most meaningful and compact representation of the dominant density pattern, with TRACE-FM offering a similarly efficient but fully connected alternative.

\begin{figure}[H]
\centering
\begin{minipage}[b]{0.24\textwidth}
    \centering
    \includegraphics[width=\textwidth]{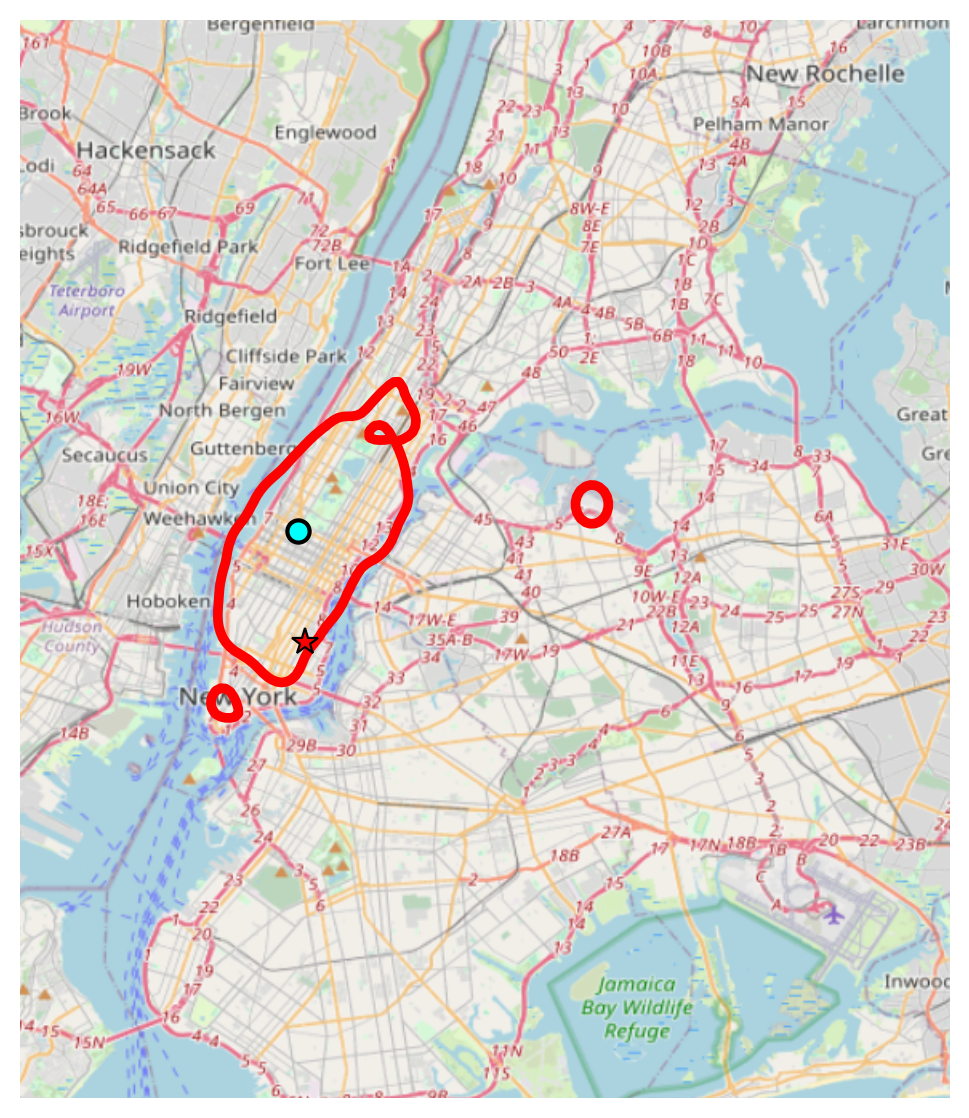}
    \centerline{\small (a) KDE}
\end{minipage}
\hfill
\begin{minipage}[b]{0.24\textwidth}
    \centering
    \includegraphics[width=\textwidth]{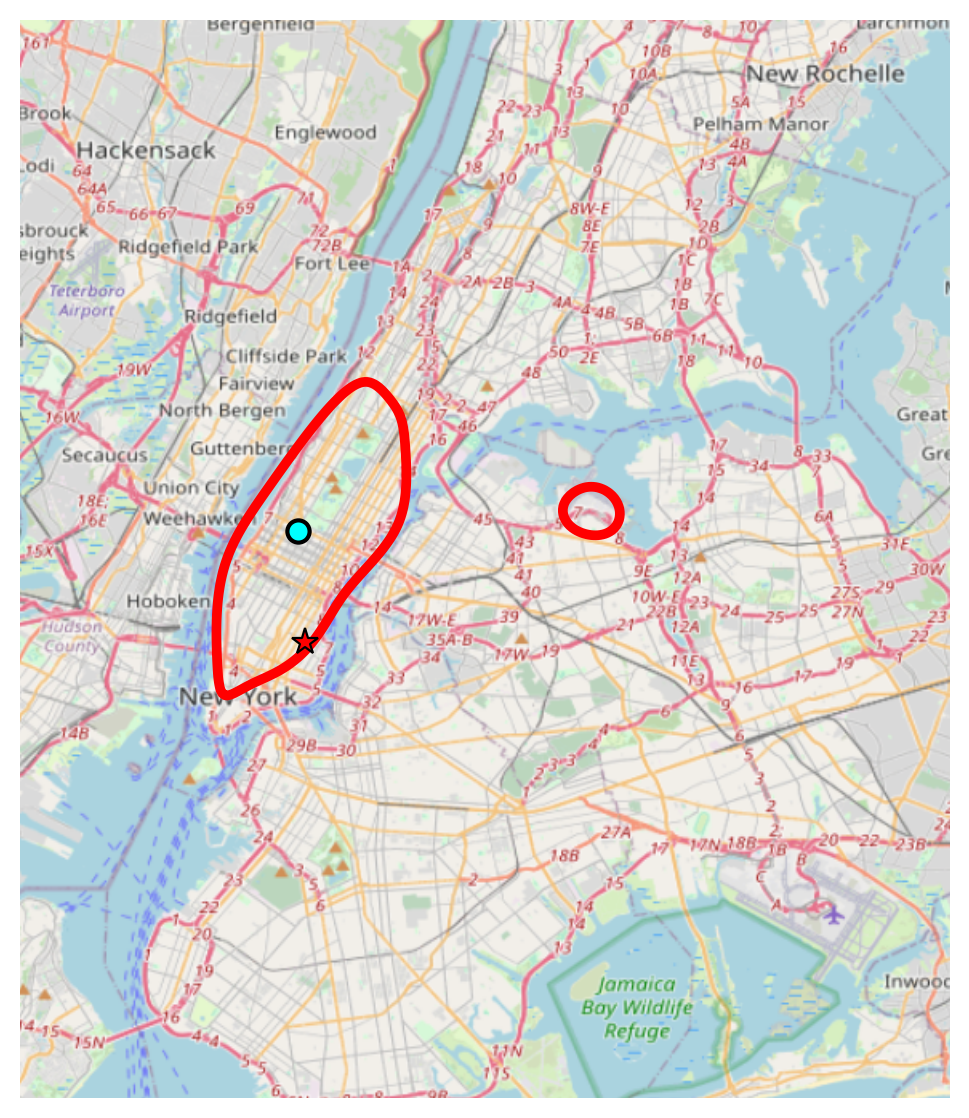}
    \centerline{\small (b) TRACE-Diff}
\end{minipage}
\hfill
\begin{minipage}[b]{0.24\textwidth}
    \centering
    \includegraphics[width=\textwidth]{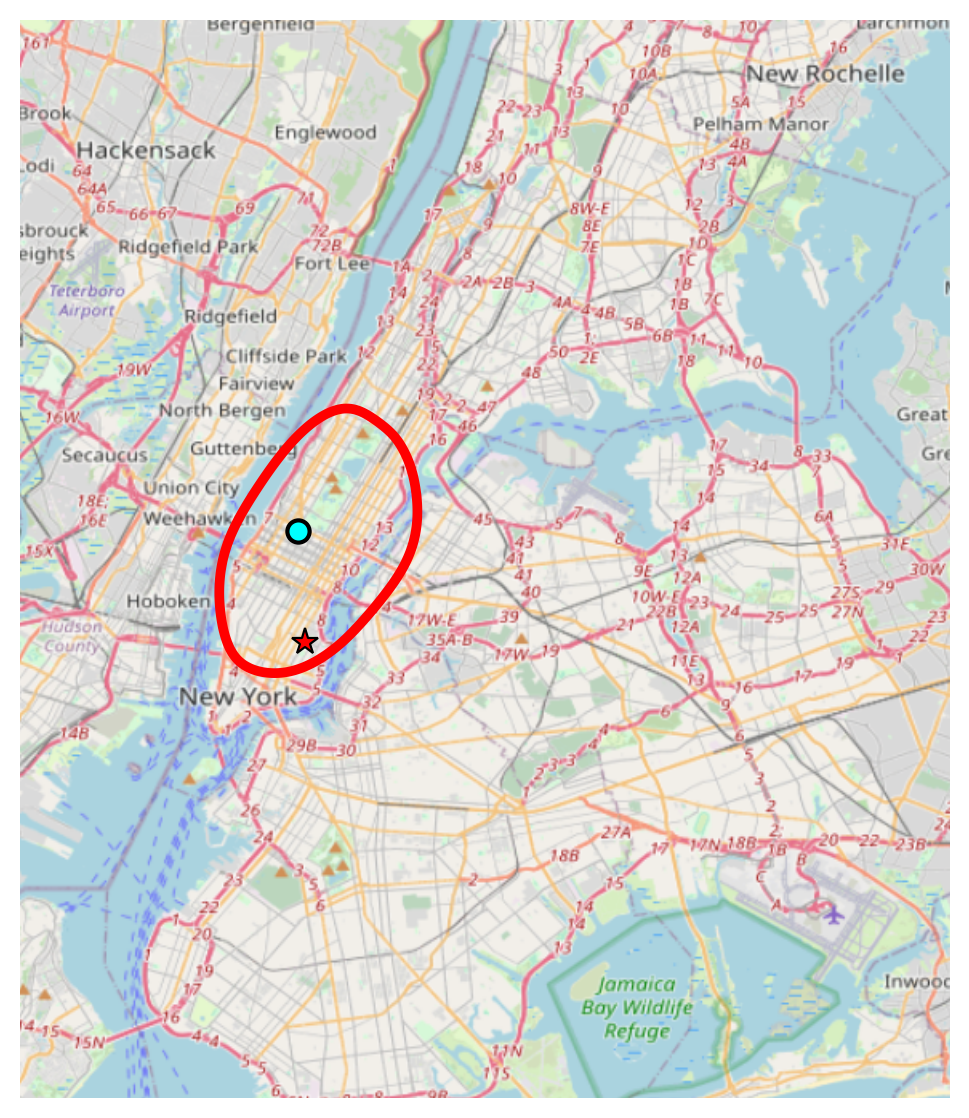}
    \centerline{\small (c) TRACE-FM}
\end{minipage}
\hfill
\begin{minipage}[b]{0.24\textwidth}
    \centering
    \includegraphics[width=\textwidth]{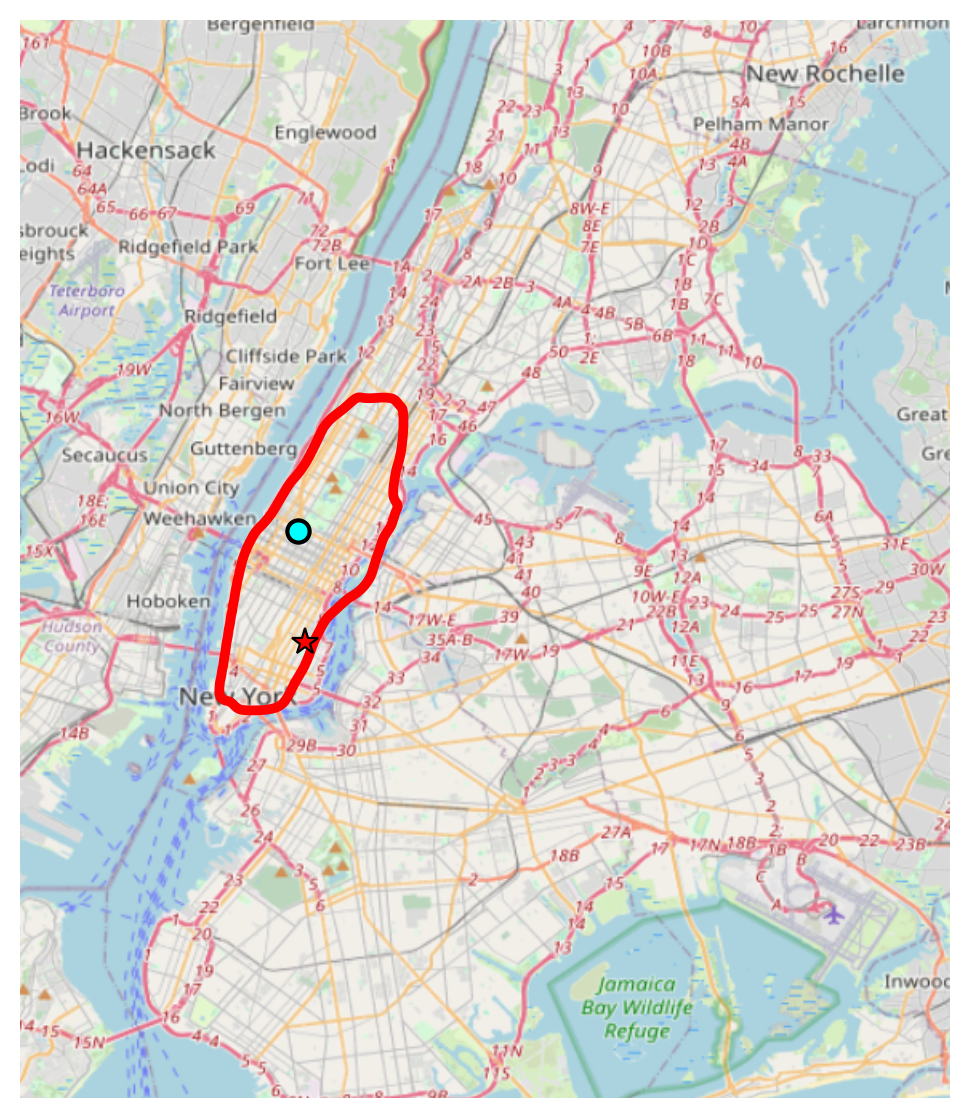}
    \centerline{\small (d) CONTRA}
\end{minipage}

\vspace{0.2cm}

\begin{minipage}[b]{0.24\textwidth}
    \centering
    \includegraphics[width=\textwidth]{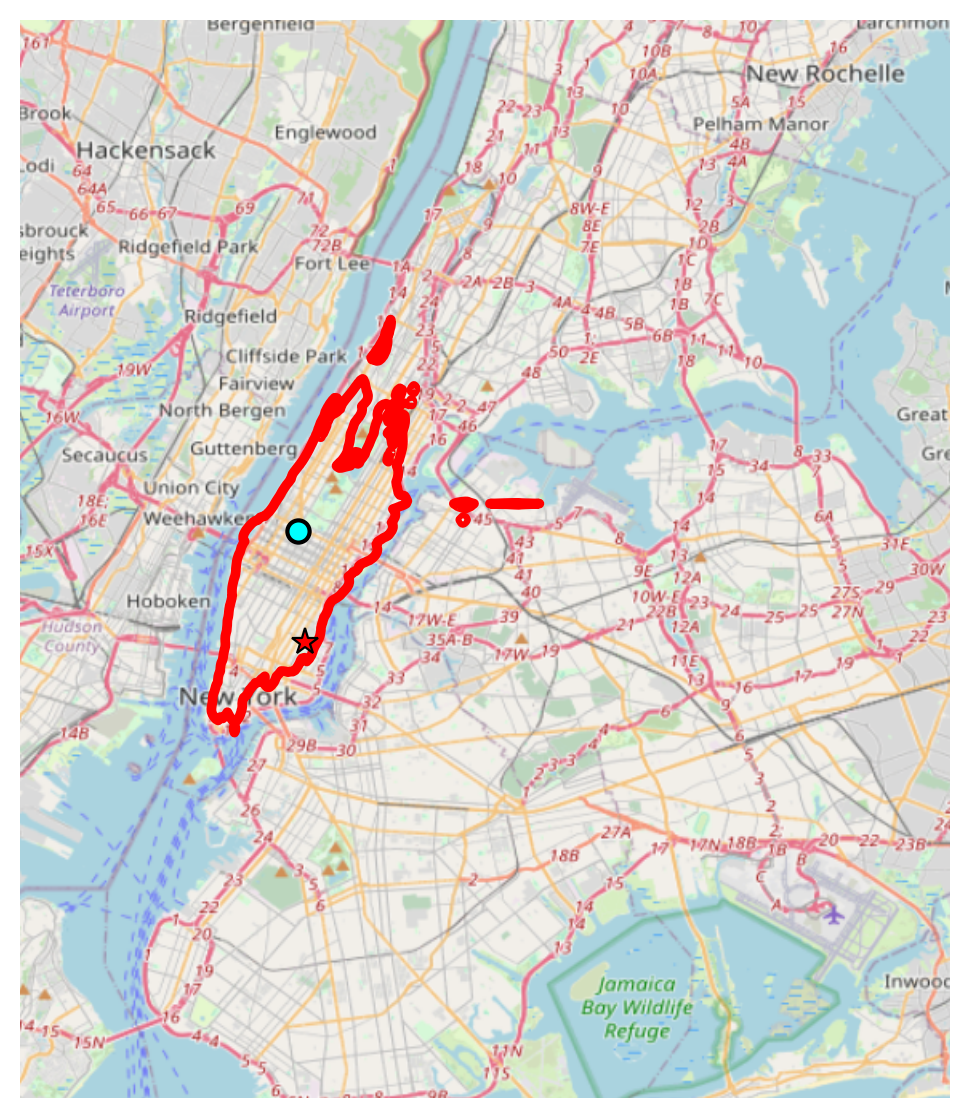}
    \centerline{\small (e) JAPAN}
\end{minipage}
\hfill
\begin{minipage}[b]{0.24\textwidth}
    \centering
    \includegraphics[width=\textwidth]{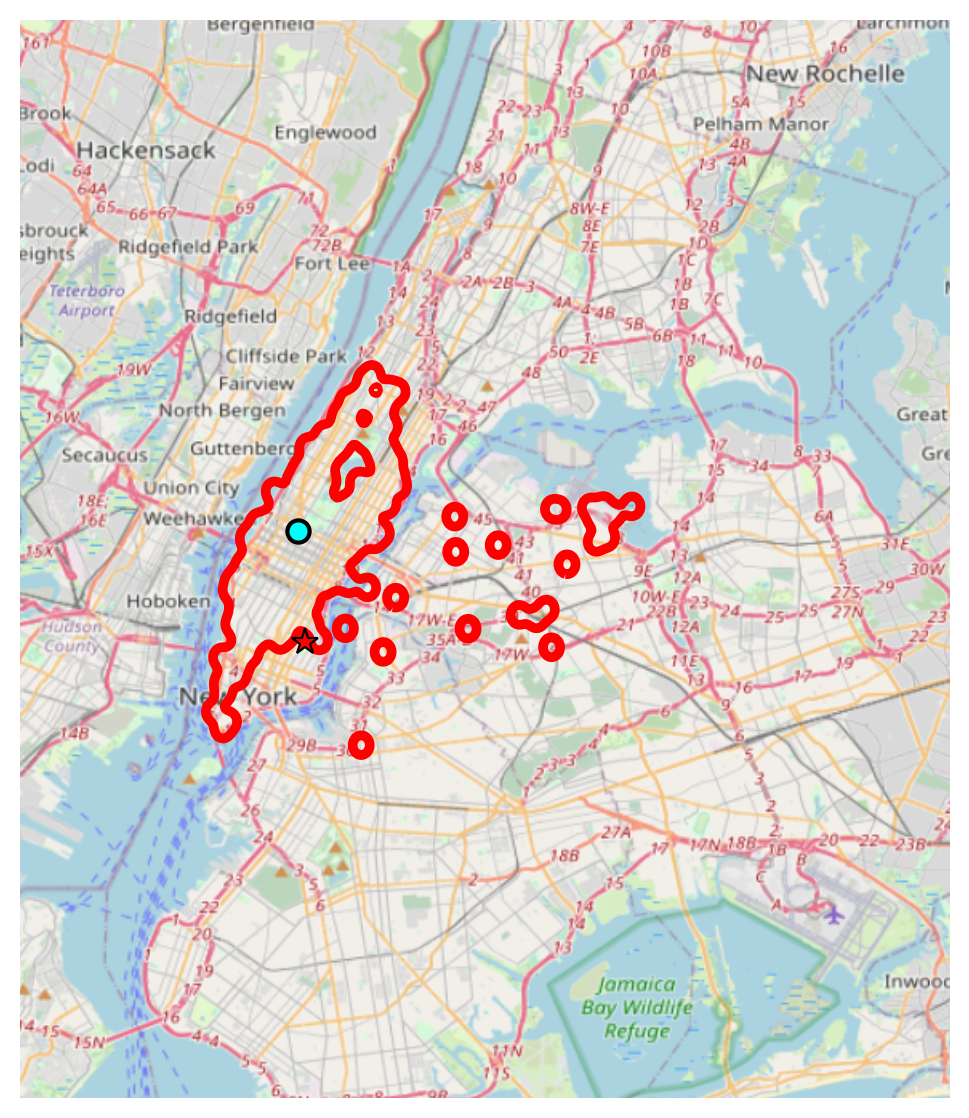}
    \centerline{\small (f) PCP-Diff}
\end{minipage}
\hfill
\begin{minipage}[b]{0.24\textwidth}
    \centering
    \includegraphics[width=\textwidth]{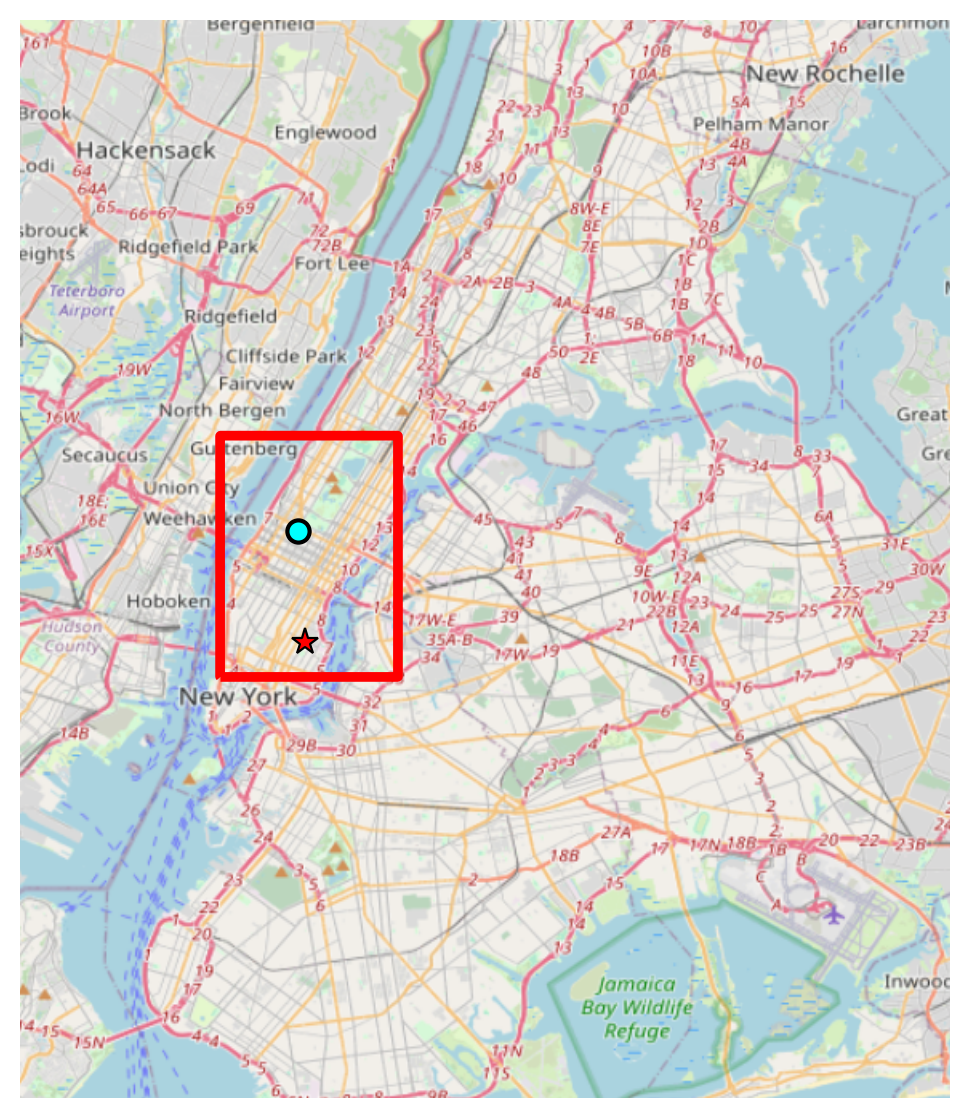}
    \centerline{\small (g) MCQR}
\end{minipage}
\hfill
\begin{minipage}[b]{0.24\textwidth}
    \centering
    \includegraphics[width=\textwidth]{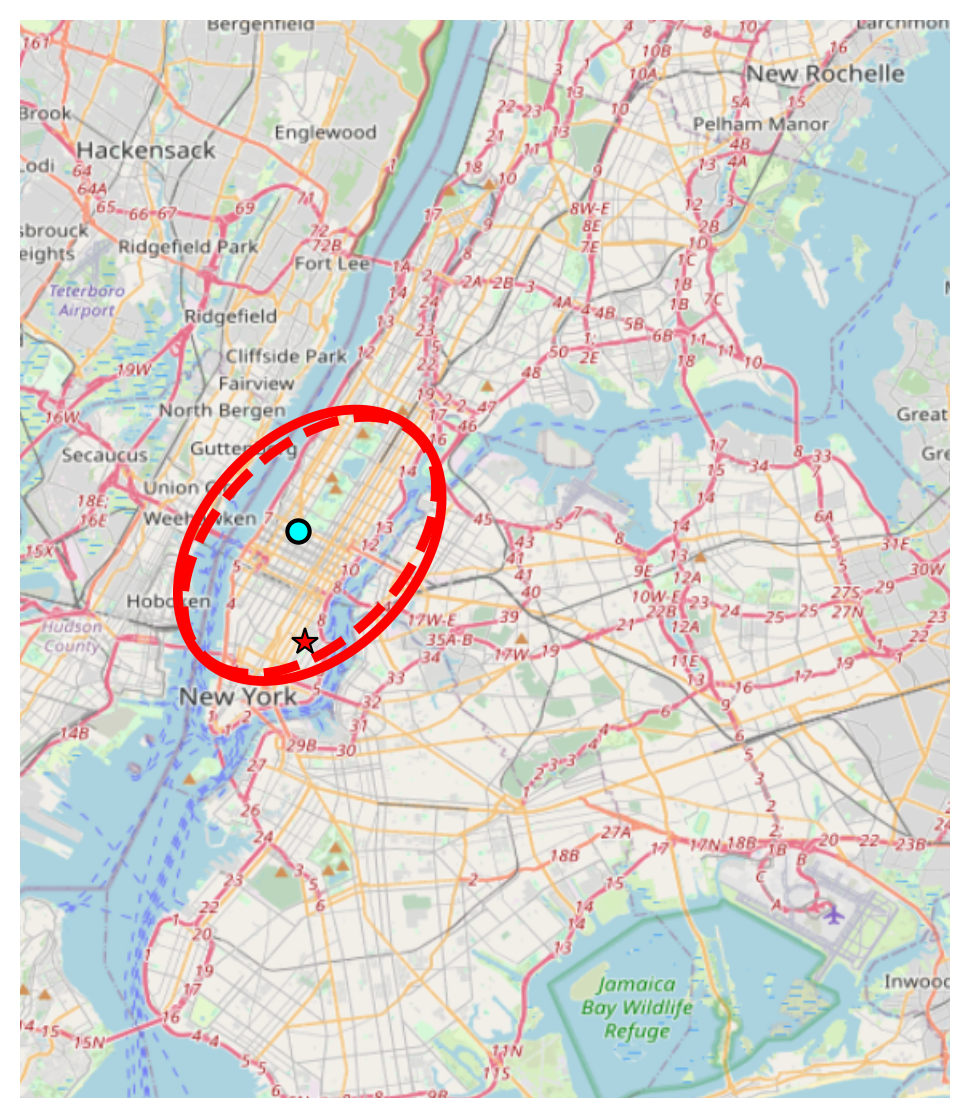}
    \centerline{\small (h) RCP / NLE}
\end{minipage}

\caption{Prediction regions on the \textsc{Taxi} dataset overlaid on a map of New York City.
The cyan dot marks the pick-up location; the red star indicates the actual drop-off for this input.
(a)~KDE reference density; (b)--(h)~conformal prediction regions for each method.}
\label{fig:regions_taxi}
\end{figure}

\begin{table}[ht]
\centering
\caption{Best volume in \textbf{bold}.
Methods with extremely large volume ($>10^6$) or standard deviation ($>10^3$)  are marked with $\dagger$.}
\label{tab:results_real}
\small
\setlength{\tabcolsep}{2.5pt}
\renewcommand{\arraystretch}{1.08}
\begin{tabular}{l cc cc cc cc}
\toprule
& \multicolumn{2}{c}{\textsc{Energy}} 
& \multicolumn{2}{c}{\textsc{RF1-2D}} 
& \multicolumn{2}{c}{\textsc{RF1-4D}} 
& \multicolumn{2}{c}{\textsc{SCM20D}} \\
\cmidrule(lr){2-3} \cmidrule(lr){4-5} \cmidrule(lr){6-7} \cmidrule(lr){8-9}
Method & Cov (\%) & Vol & Cov (\%) & Vol & Cov (\%) & Vol & Cov (\%) & Vol($\times 10^{3}$) \\
\midrule
\multicolumn{9}{l}{\emph{Transport-based (proposed)}}\\
TRACE-Diff
& 88.4 {\tiny$\pm$ 3.8} & 3.2 {\tiny$\pm$ 0.7}
& 90.8 {\tiny$\pm$ 1.1} & 1.1 {\tiny$\pm$ 0.1}
& 90.4 {\tiny$\pm$ 1.2} & 2.7 {\tiny$\pm$ 0.4}
& 89.8 {\tiny$\pm$ 2.0} & 1.8 {\tiny$\pm$ 1.4} \\

TRACE-FM
& 90.3 {\tiny$\pm$ 3.6} & \textbf{3.1} {\tiny$\pm$ 0.5}
& 90.6 {\tiny$\pm$ 1.2} & \textbf{0.9} {\tiny$\pm$ 0.1}
& 90.9 {\tiny$\pm$ 1.1} & \textbf{2.0} {\tiny$\pm$ 0.3}
& 90.1 {\tiny$\pm$ 1.5} & \textbf{0.8} {\tiny$\pm$ 0.4} \\

\addlinespace[2pt]
\multicolumn{9}{l}{\emph{Latent-density-based}}\\
CONTRA
& 90.3 {\tiny$\pm$ 3.6} & 18.1 {\tiny$\pm$ 6.6}
& 90.7 {\tiny$\pm$ 1.4} & 3.4 {\tiny$\pm$ 0.9}
& 90.2 {\tiny$\pm$ 1.4} & 596.7$^\dagger$
& 90.2 {\tiny$\pm$ 1.2} & 155.0 {\tiny$\pm$ 273.3} \\

JAPAN
& 90.8 {\tiny$\pm$ 3.8} & 15.4 {\tiny$\pm$ 6.0}
& 90.8 {\tiny$\pm$ 1.3} & 2.4 {\tiny$\pm$ 0.4}
& 89.9 {\tiny$\pm$ 1.4} & 10.6 {\tiny$\pm$ 3.0}
& 90.2 {\tiny$\pm$ 1.2} & 153.5 {\tiny$\pm$ 271.5} \\

\addlinespace[2pt]
\multicolumn{9}{l}{\emph{Sample-based}}\\
PCP-Diff
& 89.5 {\tiny$\pm$ 4.1} & 4.0 {\tiny$\pm$ 0.8}
& 90.7 {\tiny$\pm$ 1.2} & 1.2 {\tiny$\pm$ 0.1}
& 90.4 {\tiny$\pm$ 1.0} & 3.2 {\tiny$\pm$ 0.3}
& 90.0 {\tiny$\pm$ 1.7} & 73.7 {\tiny$\pm$ 5.8} \\

\addlinespace[2pt]
\multicolumn{9}{l}{\emph{Shape-restricted}}\\
RCP
& 90.7 {\tiny$\pm$ 5.1} & 40.7 {\tiny$\pm$ 16.9}
& 90.1 {\tiny$\pm$ 1.2} & 578.4$^\dagger$ 
& 90.0 {\tiny$\pm$ 1.0} & $\dagger$
& 90.3 {\tiny$\pm$ 1.9} & $\dagger$ \\

NLE
& 91.1 {\tiny$\pm$ 4.3} & 19.2 {\tiny$\pm$ 5.9}
& 90.3 {\tiny$\pm$ 1.0} & 154.2{\tiny$\pm$ 602.3} 
& 90.0 {\tiny$\pm$ 1.2} & $\dagger$
& 90.0 {\tiny$\pm$ 1.7} & $\dagger$ \\

MCQR
& 87.3 {\tiny$\pm$ 4.8} & 24.0 {\tiny$\pm$ 7.9}
& 90.2 {\tiny$\pm$ 1.2} & 24.3 {\tiny$\pm$ 92.1}
& 89.4 {\tiny$\pm$ 1.1} & $\dagger$
& 89.3 {\tiny$\pm$ 1.5} & 76.3 {\tiny$\pm$ 19.7} \\
\bottomrule
\end{tabular}
\end{table}

\paragraph{Energy, RF1, and SCM20D.}
On \textsc{Energy}, TRACE yields the smallest regions,
with TRACE-FM and TRACE-Diff achieving volumes of $3.1$
and $3.2$, respectively. These are approximately five to
six times smaller than JAPAN and CONTRA.
The advantage persists on \textsc{RF1-2D} and becomes more
pronounced on \textsc{RF1-4D}, where TRACE-FM attains a
volume of $2.0$, compared to $10.6$ for JAPAN and
substantially larger and highly variable regions for CONTRA.
The contrast is most striking on \textsc{SCM20D}.
Although the target remains two-dimensional, the input dimension
is substantially higher.
Here TRACE-FM produces a region of size $0.8 \times 10^{3}$,
whereas both CONTRA and JAPAN exceed $150 \times 10^{3}$,
with extreme cross-split variability.
Even PCP-Diff and MCQR remain roughly two orders of
magnitude larger than TRACE-FM.
Across all datasets, coverage remains close to the nominal
$90\%$ level, while TRACE consistently yields the most
compact regions.
The $\dagger$ entries in Table~\ref{tab:results_real}
reflect two distinct failure modes: shape-restricted methods
enclose substantial empty space due to rigid geometric
templates, while latent-density methods lose precision on
fine-grained conditional structure as input or output
dimensionality grows.


\subsection{Stability Analysis}\label{sec:stability}

\begin{figure}[ht]
\centering
\begin{minipage}[b]{0.48\textwidth}
    \centering
    \includegraphics[width=\textwidth, trim=0 0 0 20, clip]{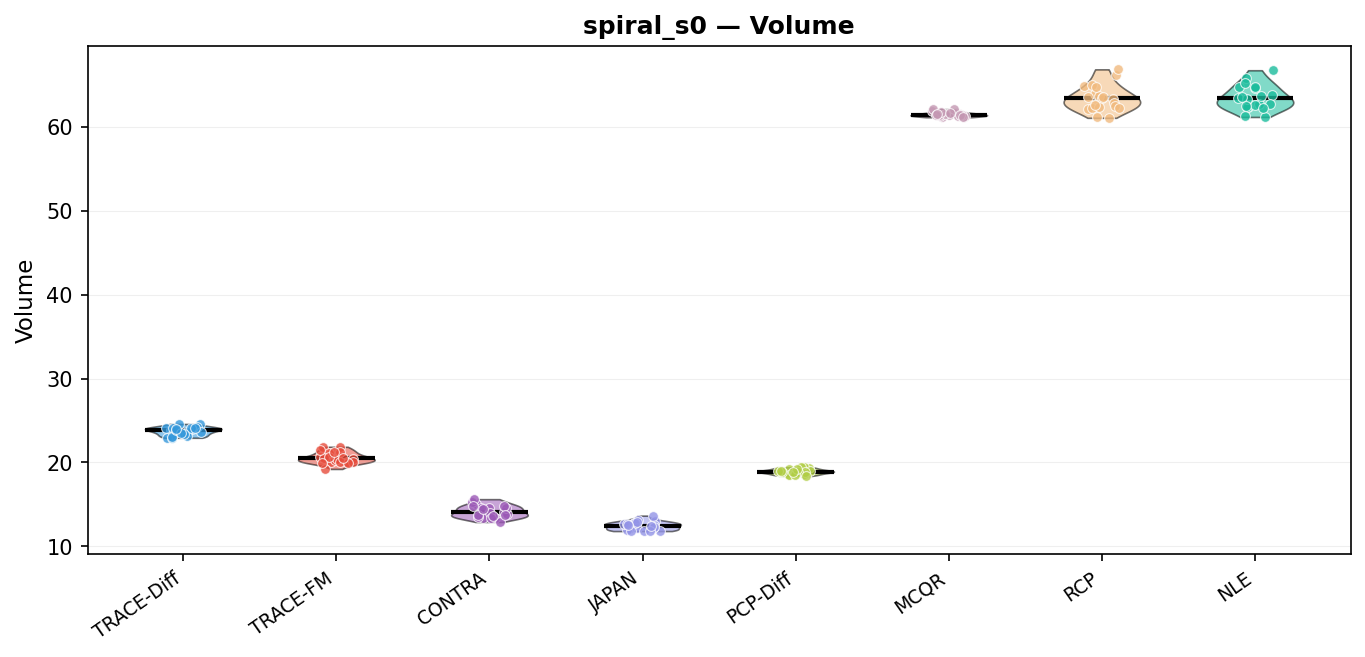}
    \centerline{\small (a) \textsc{Spiral$_L$}}
\end{minipage}
\hfill
\begin{minipage}[b]{0.48\textwidth}
    \centering
    \includegraphics[width=\textwidth, trim=0 0 0 20, clip]{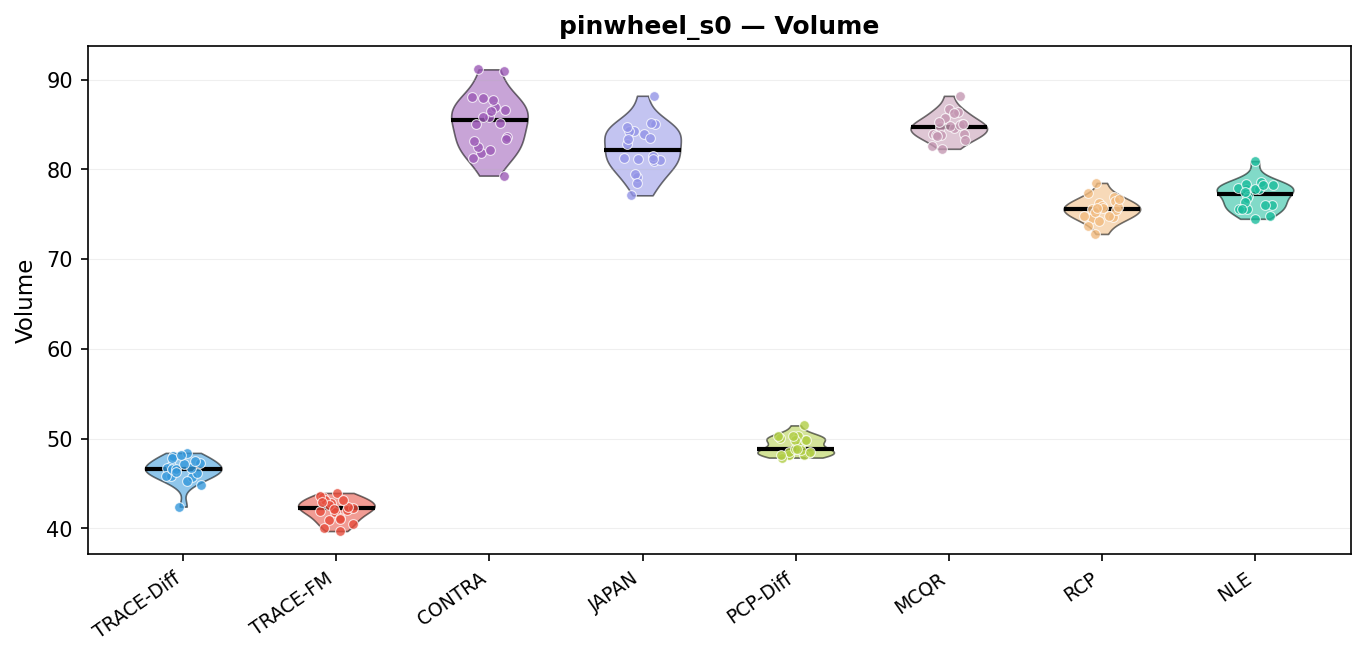}
    \centerline{\small (b) \textsc{Pinwheel$_H$}}
\end{minipage}

\vspace{0.3cm}

\begin{minipage}[b]{0.48\textwidth}
    \centering
    \includegraphics[width=\textwidth, trim=0 0 0 20, clip]{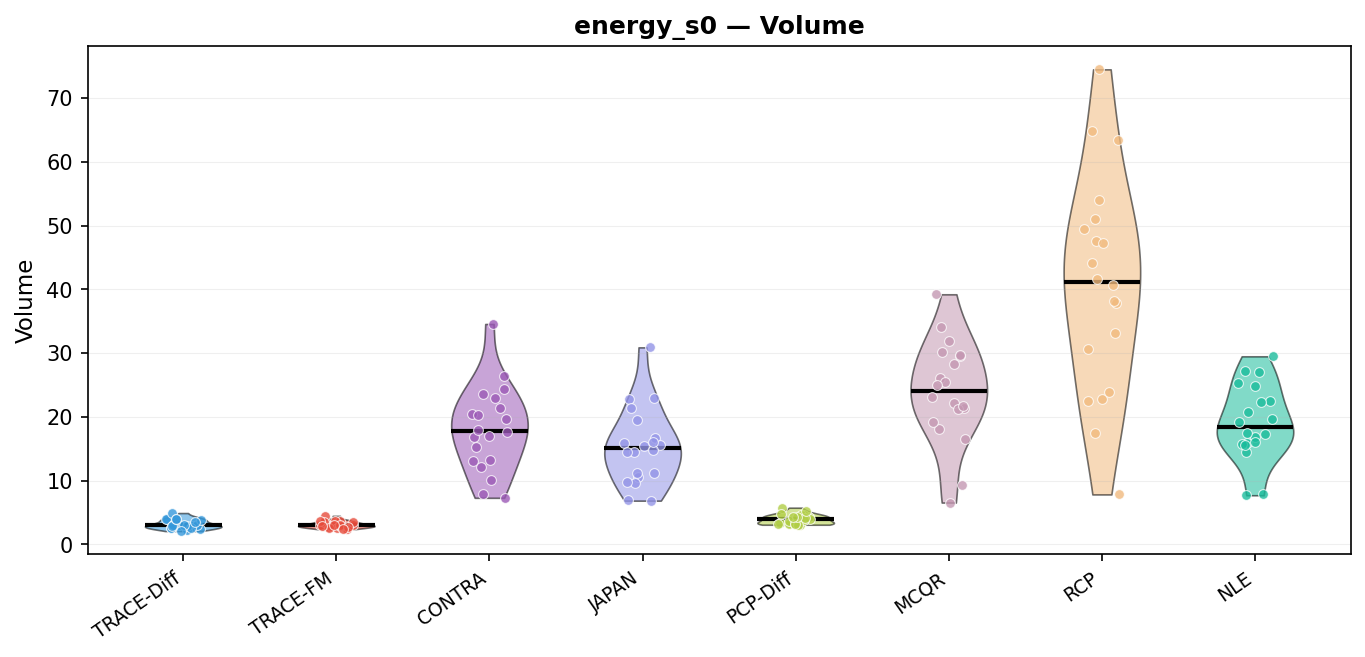}
    \centerline{\small (c) \textsc{Energy}}
\end{minipage}
\hfill
\begin{minipage}[b]{0.48\textwidth}
    \centering
    \includegraphics[width=\textwidth, trim=0 0 0 20, clip]{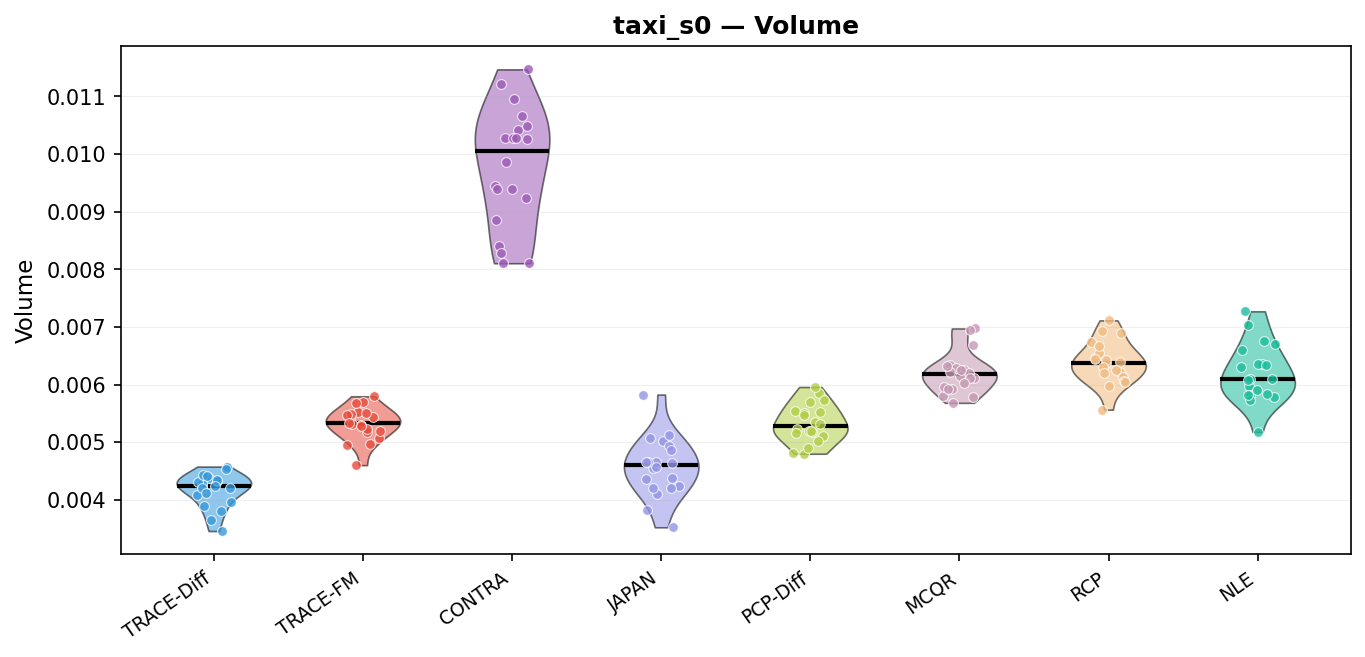}
    \centerline{\small (d) \textsc{Taxi}}
\end{minipage}
\caption{Prediction region volume across 20 repeated experiments on two synthetic and two real-world datasets.
Each dot represents one repeat; the horizontal bar indicates the mean.}
\label{fig:volume_violin}
\end{figure}
Besides average performance, the variability of region
volume across random splits is also crucial for practical
reliability. Figure~\ref{fig:volume_violin} displays the distribution of region volumes
over 20 random repeats for four representative datasets.
While the ranking of the mean performance of different
methods is the same as that reported in the tables,
the violin plots highlight important differences in cross-split stability.

Across all datasets shown, TRACE exhibits tightly
concentrated distributions with relatively low variance,
indicating robustness to data splits and random initialization.
In contrast, latent-density-based methods display noticeably wider spreads
on more challenging datasets, particularly \textsc{Energy},
suggesting greater sensitivity to training randomness.
Shape-restricted baselines show mixed stability behavior:
in some cases their distributions remain concentrated,
while in others they exhibit substantial variability,
and their volumes are consistently larger overall.

In short, the stability analysis indicates that
TRACE not only achieves compact regions
on average, but also maintains reliable performance
across repeated random splits.

\subsection{Monte Carlo Budget}
\label{sec:ablation_mc}
\begin{figure}[ht]
\centering
\includegraphics[width=\textwidth]{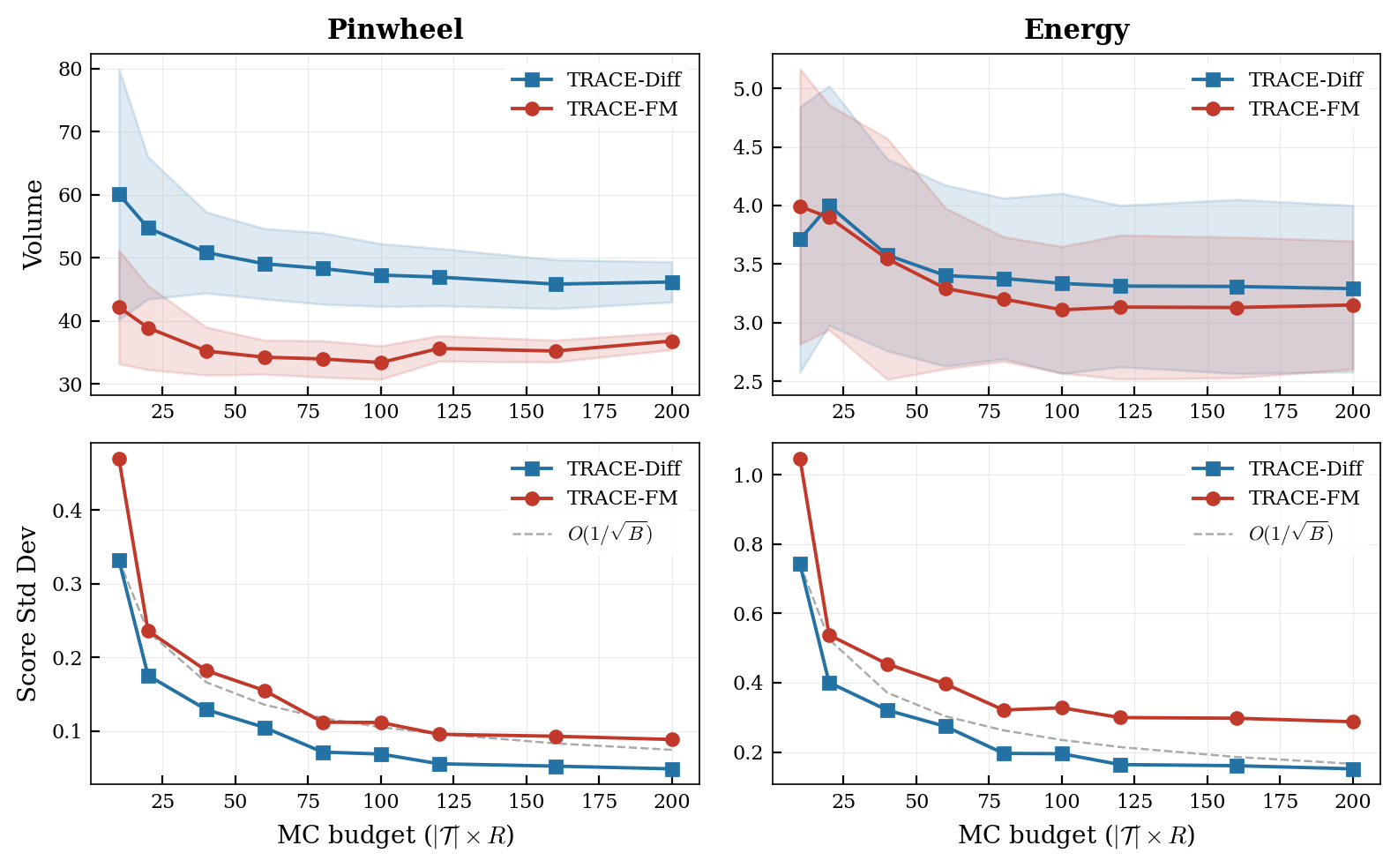}
\caption{
Effect of Monte Carlo budget $B = |\mathcal{T}| \times R$ on conformal region volume (top row) 
and the standard deviation of  TRACE evaluated on calibration points (bottom row)
for TRACE-Diff and TRACE-FM on \textsc{Pinwheel} (left) and \textsc{Energy} (right).
Results are averaged over 20 repeats; solid lines denote the mean and shaded bands indicate $\pm 1$ standard deviation.The dashed gray line shows the $O(B^{-1/2})$ reference decay for score standard deviation.
}
\label{fig:mc_ablation}
\end{figure}

Section~\ref{sec:theory} established that Monte Carlo noise in TRACE scores inflates the calibrated quantile and, consequently, the prediction region volume.
We empirically examine this relationship by varying the total Monte Carlo budget
$B = |\mathcal{T}| \times R$ and measuring its effect on calibration score variability and conformal region volume.

Figure~\ref{fig:mc_ablation} reveals three consistent patterns. First, the standard deviation of TRACE scores across calibration points decreases monotonically with increasing budget on both datasets, closely following the $O(B^{-1/2})$ reference curve.
This empirical behavior aligns with the Monte Carlo variance analysis in Section~\ref{sec:theory}. Second, prediction region volumes decrease overall as $B$ increases, though not strictly monotonically.
While mild fluctuations appear at intermediate budgets due to finite-sample variability,
the curves approach a stable plateau for larger $B$, indicating diminishing returns from additional Monte Carlo samples. Third, beyond $B \approx 100\text{--}120$, both calibration score variability and conformal region volume change only marginally.
The default configuration used in the main experiments ($|\mathcal{T}| = 15$, $R = 8$, yielding $B = 120$) lies within this stable regime.
This suggests that the volumes reported in Tables~\ref{tab:results_synthetic} and~\ref{tab:results_real}
are not materially driven by residual MC estimation noise under our experimental setup.


\section{Conclusion}
\label{sec:conclusion}

We introduced TRACE, a transport-based nonconformity scoring framework derived directly from the training objectives of conditional diffusion and flow matching models.
By aggregating denoising or velocity-matching errors across multiple time steps and perturbations, and using a CRN bank shared across all evaluations, our scores define a deterministic conformal ordering \emph{conditional on the fixed bank} and plug seamlessly into split conformal prediction with finite-sample marginal coverage guarantees.
Across nine datasets, the proposed approach, especially TRACE-FM, yields prediction regions with competitive or smaller volume than latent-density and shape-restricted baselines, while naturally adapting to curved, multimodal, and disconnected conditional distributions.
Monte Carlo approximation primarily affects efficiency (e.g., calibrated thresholds and region volume) rather than validity; for any fixed CRN bank, the conformal coverage guarantee holds exactly.
The main computational overhead is evaluation-time averaging, which scales linearly with the Monte Carlo budget and can be efficiently parallelized and batched on modern hardware.

Several directions remain open.
Extending transport-based scoring to structured outputs (e.g., sequences, graphs, and images) would broaden applicability beyond Euclidean targets.
Since region quality depends on the underlying generative model, understanding robustness under model misspecification is an important next step.
Developing variants that target stronger notions such as conditional coverage would further enhance practical utility.
Another direction is to combine TRACE with
user-specified point estimators by applying transport
alignment to residuals, extending the framework beyond
fully generative models.



\newpage


\appendix

\section{ Proof of Theorem~\ref{thm:score_mse}}
\label{app:proof_score_mse}
\begin{proof}
For each fixed $t \in \mathcal{T}$, define
$$
\hat{\mu}_t
=
\frac{1}{R}
\sum_{r=1}^{R}
\ell(x,y;t,\xi_{t,r}),
\qquad
\mu_t
=
\mathbb{E}_{\xi}
\bigl[
\ell(x,y;t,\xi)
\bigr].
$$
Then
$$
s_\omega(x,y)
=
\frac{1}{|\mathcal{T}|}
\sum_{t \in \mathcal{T}}
\hat{\mu}_t,
\qquad
\bar{s}_{\mathcal{T}}(x,y)
=
\frac{1}{|\mathcal{T}|}
\sum_{t \in \mathcal{T}}
\mu_t.
$$
By linearity of expectation, 
$
\mathbb{E}_\omega[\hat{\mu}_t]
=
\mu_t$. Therefore,
$$
\mathbb{E}_\omega[s_\omega(x,y)]
=
\frac{1}{|\mathcal{T}|}
\sum_{t \in \mathcal{T}}
\mu_t
=
\bar{s}_{\mathcal{T}}(x,y).
$$
Further, independence of the auxiliary draws across $r$ and $t$ implies that 
$$
\mathrm{Var}_\omega(\hat{\mu}_t)
=
\frac{1}{R}
\mathrm{Var}_{\xi}
\bigl(
\ell(x,y;t,\xi)
\bigr),
$$
and
$$
\mathrm{Var}_\omega(s_\omega)
=
\mathrm{Var}_\omega\!\left(
\frac{1}{|\mathcal{T}|}
\sum_{t \in \mathcal{T}}
\hat{\mu}_t
\right)
=
\frac{1}{|\mathcal{T}|^2}
\sum_{t \in \mathcal{T}}
\mathrm{Var}_\omega(\hat{\mu}_t)
=
\frac{1}{|\mathcal{T}|^2}
\sum_{t \in \mathcal{T}}
\frac{
\mathrm{Var}_{\xi}
\bigl(
\ell(x,y;t,\xi)
\bigr)
}{R}.
$$
Finally, if $
\sup_{t\in\mathcal{T}}
\mathrm{Var}_{\xi}
\bigl(
\ell(x,y;t,\xi)
\bigr)
\le
C
$, 
then
$$
\mathbb{E}_\omega
\bigl[
(
s_\omega(x,y)
-
\bar{s}_{\mathcal{T}}(x,y)
)^2
\bigr]=\mathrm{Var}_\omega(s_\omega)
\le
\frac{|\mathcal{T}|}{|\mathcal{T}|^2}
\frac{C}{R}
=
\frac{C}{|\mathcal{T}|R}
=
O\!\left(\frac{1}{B}\right).
$$
And by the Cauchy--Schwarz inequality,
$$
\mathbb{E}_\omega
\bigl[\bigl|
s_\omega(x,y)
-
\bar{s}_{\mathcal{T}}(x,y)
\bigr|\bigr]
=
O\!\left(\frac{1}{\sqrt{B}}\right).
$$
\end{proof}

\section{Proof of Theorem~\ref{thm:threshold_rate}}
\label{app:proof_threshold}
\begin{proof}
Let the calibration set be $\{(x_i,y_i)\}_{i=1}^{n_{\mathrm{cal}}}$.
Define
$$
Z_i := \bar{s}_{\mathcal{T}}(x_i,y_i),
\qquad
\tilde Z_i := s_\omega(x_i,y_i),
\qquad
\varepsilon_i := \tilde Z_i - Z_i .
$$
Let $k = \left\lceil (1-\alpha)(n_{\mathrm{cal}}+1) \right\rceil$.
Denote by $q^{\mathcal{T}}$ and $q^\omega$
the $k$-th order statistics of $\{Z_i\}_{i=1}^{n_{\mathrm{cal}}}$ and
$\{\tilde Z_i\}_{i=1}^{n_{\mathrm{cal}}}$, respectively.

\subsection*{Step 1: Deterministic order-statistic stability}

For any two real vectors $\{a_i\}_{i=1}^n$ and $\{b_i\}_{i=1}^n$,
let $a_{(k)}$ and $b_{(k)}$ denote their $k$-th order statistics.
Then
$$
|a_{(k)} - b_{(k)}|
\le
\max_{1\le i \le n} |a_i - b_i|.
$$
Applying this with $a_i = Z_i$ and $b_i = \tilde Z_i$ yields
$$
|q^\omega - q^{\mathcal{T}}|
\le
\max_{1\le i \le n_{\mathrm{cal}}} |\varepsilon_i|.
$$
Therefore,
$$
\mathbb{E}_\omega\bigl[|q^\omega - q^{\mathcal{T}}|\bigr]
\le
\mathbb{E}_\omega\Bigl[\max_{1\le i \le n_{\mathrm{cal}}} |\varepsilon_i|\Bigr].
$$

\subsection*{Step 2: Bounding $\mathbb{E}_\omega[\max_i |\varepsilon_i|]$}

By Theorem~\ref{thm:score_mse}, for each $i$,
$$
\mathbb{E}_\omega[\varepsilon_i^2]
=
\mathbb{E}_\omega\!\left[
\bigl(s_\omega(x_i,y_i)-\bar{s}_{\mathcal{T}}(x_i,y_i)\bigr)^2
\right]
\le
\frac{C}{B},
$$
where
$$
C := \max_{1 \le i \le n_{\mathrm{cal}}} \sup_{t \in \mathcal{T}}
\mathrm{Var}_\xi\bigl(\ell(x_i,y_i;t,\xi)\bigr),
\qquad
B = |\mathcal{T}|\cdot R.
$$
Let $W := \max_{1\le i\le n_{\mathrm{cal}}} |\varepsilon_i| \ge 0$.
Using the tail-integral representation,
$$
\mathbb{E}_\omega[W]
=
\int_0^\infty \mathbb{P}_\omega(W>t)\,dt.
$$
By the union bound,
$$
\mathbb{P}_\omega(W>t)
=
\mathbb{P}_\omega\!\left(\max_i |\varepsilon_i|>t\right)
\le
\sum_{i=1}^{n_{\mathrm{cal}}}
\mathbb{P}_\omega(|\varepsilon_i|>t).
$$
By Chebyshev's inequality,
$$
\mathbb{P}_\omega(|\varepsilon_i|>t)
\le
\frac{\mathbb{E}_\omega[\varepsilon_i^2]}{t^2}
\le
\frac{C}{B\,t^2}.
$$
Combining the last two displays gives
$$
\mathbb{P}_\omega(W>t)
\le
\frac{n_{\mathrm{cal}}\, C}{B\, t^2}.
$$
Let $t_0 := \sqrt{n_{\mathrm{cal}}\, C / B}$.
Splitting the tail integral at $t_0$ and using $\mathbb{P}(\cdot)\le 1$ on $[0,t_0]$,
\begin{align*}
\mathbb{E}_\omega[W]
&\le
\int_0^{t_0} 1\,dt
+
\int_{t_0}^{\infty}
\frac{n_{\mathrm{cal}}\, C}{B\, t^2}\,dt
\\
&=
t_0
+
\frac{n_{\mathrm{cal}}\, C}{B}\cdot \frac{1}{t_0}
=
2\sqrt{\frac{n_{\mathrm{cal}}\, C}{B}}.
\end{align*}
Putting Steps~1 and 2 together yields the claimed bound:
$$
\mathbb{E}_\omega\bigl[|q^\omega - q^{\mathcal{T}}|\bigr]
\le
\mathbb{E}_\omega\Bigl[\max_i |\varepsilon_i|\Bigr]
\le
2\sqrt{\frac{n_{\mathrm{cal}}\, C}{B}}\,.
$$

\end{proof}

\section{Proof of Proposition~\ref{prop:discretization}}
\label{app:proof_discretization}

\begin{proof}
Fix $(x,y)$ and define
$$
\mu(t) := \mu(t;x,y) := \mathbb{E}_{\xi}\bigl[\ell(x,y;t,\xi)\bigr].
$$
By Assumption~\ref{assump:lipschitz_time}, $\mu$ is $L$-Lipschitz on $[0,1]$, i.e.,
$$
|\mu(t_1)-\mu(t_2)| \le L|t_1-t_2|
\qquad \forall\, t_1,t_2\in[0,1].
$$
Recall that
$$
\bar{s}(x,y) = \int_0^1 \mu(t)\,dt,
\qquad
\bar{s}_{\mathcal{T}}(x,y) = \frac{1}{m}\sum_{j=1}^m \mu(t_j).
$$
Partition $[0,1]$ into $m$ subintervals of length $\Delta$:
$$
I_j := \bigl[(j-1)\Delta,\, j\Delta\bigr], \qquad j=1,\dots,m.
$$
Then
$$
\bar{s}(x,y) - \bar{s}_{\mathcal{T}}(x,y)
=
\sum_{j=1}^m
\left(
\int_{I_j} \mu(t)\,dt
-
\Delta\,\mu(t_j)
\right)
=
\sum_{j=1}^m
\int_{I_j}\bigl(\mu(t)-\mu(t_j)\bigr)\,dt.
$$
Taking absolute values and applying the triangle inequality yields
$$
\bigl|\bar{s}(x,y) - \bar{s}_{\mathcal{T}}(x,y)\bigr|
\le
\sum_{j=1}^m
\int_{I_j}
\bigl|\mu(t)-\mu(t_j)\bigr|\,dt.
$$
By Lipschitz continuity,
$$
\bigl|\mu(t)-\mu(t_j)\bigr|
\le
L\,|t-t_j|
\qquad \forall\, t\in I_j.
$$
Since $t_j\in I_j$ and $I_j$ has length $\Delta$, we have
$$
|t-t_j| \le \Delta \qquad \forall\, t\in I_j,
$$
and hence
$$
\int_{I_j}\bigl|\mu(t)-\mu(t_j)\bigr|\,dt
\le
\int_{I_j} L\,|t-t_j|\,dt
\le
\int_{I_j} L\,\Delta\,dt
=
L\,\Delta^2.
$$
Summing over $j=1,\dots,m$ gives
$$
\bigl|\bar{s}(x,y) - \bar{s}_{\mathcal{T}}(x,y)\bigr|
\le
\sum_{j=1}^m L\,\Delta^2
=
mL\,\Delta^2
=
L\,\Delta
=
\frac{L}{m}.
$$

To obtain the stated constant $1/2$, observe that for any interval $I_j$,
the average distance from a fixed point $t_j\in I_j$ to points $t\in I_j$
is at most $\Delta/2$, i.e.,
$$
\int_{I_j} |t-t_j|\,dt \le \frac{\Delta^2}{2}.
$$
Indeed, the integral is maximized when $t_j$ is at an endpoint of $I_j$,
in which case $\int_{0}^{\Delta} u\,du = \Delta^2/2$.
Therefore,
$$
\int_{I_j}\bigl|\mu(t)-\mu(t_j)\bigr|\,dt
\le
L \int_{I_j}|t-t_j|\,dt
\le
L \cdot \frac{\Delta^2}{2}.
$$
Substituting this sharper bound into the earlier sum yields
$$
\bigl|\bar{s}(x,y) - \bar{s}_{\mathcal{T}}(x,y)\bigr|
\le
\sum_{j=1}^m \frac{L\Delta^2}{2}
=
\frac{mL\Delta^2}{2}
=
\frac{L}{2m}
=
\frac{L}{2|\mathcal{T}|}.
$$
This completes the proof.
\end{proof}

\section{Synthetic Dataset Construction}
\label{app:synthetic_data}

All four synthetic datasets generate outputs as
$Y = k \cdot f(x_1, x_2) + \varepsilon$, where $f$ is a
shared nonlinear function, $k > 0$ is a signal scale,
and $\varepsilon$ is a structured noise term whose
distribution defines the geometry of
$p(Y \mid X = x)$.
The $2 \times 2$ factorial design crosses two noise
types (spiral, pinwheel) with two conditioning regimes
(lower-dimensional, higher-dimensional), as summarized in
Table~\ref{tab:2x2_design}.

\paragraph{Shared conditional mean function.}
All four datasets use the map
$f: \mathbb{R}^2 \to \mathbb{R}^2$,
$$
f(x_1, x_2) =
\begin{pmatrix}
2x_1^3 - 3x_2^2 + 5x_2 + x_1 x_2 \\
x_1^2 x_2 - 4x_2^2 + 3x_1^2 x_2 + 7
\end{pmatrix}.
$$
Only the first two coordinates of $X$ enter $f$;
any remaining dimensions are nuisance covariates.

\paragraph{Conditioning regimes.}
The lower-dimensional variants (\textsc{Spiral$_L$},
\textsc{Pinwheel$_L$}) use
$X \sim \mathcal{N}((-2, -1.5)^\top, I_2)$ with signal
scale $k = 1$, so the conditional mean is $g(X) = f(X)$
and the signal and noise contribute comparably to $Y$.
The higher-dimensional variants (\textsc{Spiral$_H$},
\textsc{Pinwheel$_H$}) use
$X \sim \mathcal{N}(0, I_7)$ with signal scale $k = 5$,
so the conditional mean is $g(X) = 5 \cdot f(x_1, x_2)$
and the signal dominates.
The five additional input dimensions in the
higher-dimensional regime carry no information about $Y$
but must be processed by the conditioning network,
testing robustness to irrelevant input features.

\paragraph{Spiral noise.}
The noise $\varepsilon$ is generated by first drawing an
angle $\theta \sim \mathrm{Uniform}(0, 2\pi)$ and then
sampling
$$
\varepsilon =
\begin{pmatrix}
\theta \cos\theta + \eta_1 \\
\theta \sin\theta + \eta_2
\end{pmatrix},
\qquad
\eta_1 \sim \mathcal{N}(0, 0.2^2),\;
\eta_2 \sim \mathcal{N}(0, 0.1^2).
$$
The noise traces a spiral arc around the origin, so that
for a given $x$ the conditional $p(Y \mid X = x)$ is a
single curved band centered at $k \cdot f(x_1, x_2)$.
The distribution is unimodal, connected, and non-convex.

\paragraph{Pinwheel noise.}
The noise $\varepsilon$ is drawn from a mixture of
$K = 6$ Gaussian components arranged in a regular hexagon.
Component $k$ ($k = 0, \dots, 5$) has mean and covariance
$$
\mu_k = R
\begin{pmatrix}
\cos\theta_k \\ \sin\theta_k
\end{pmatrix},
\qquad
\Sigma_k = Q_k
\begin{pmatrix}
1 & 0 \\
0 & e^2
\end{pmatrix}
Q_k^\top,
$$
where $\theta_k = 2\pi k / 6$,
$Q_k$ is the rotation matrix by angle $\theta_k$,
$R = 3$ is the hexagon radius, and $e = 0.16$ controls
the eccentricity of each elongated ellipse.
Components are selected with equal probability $1/6$.
For a given $x$, the conditional $p(Y \mid X = x)$
consists of six well-separated elongated clusters with
six-fold rotational symmetry.
The distribution is multimodal and non-convex, with deep
concavities between arms that convex or ellipsoidal
regions cannot capture without enclosing substantial
empty space.

\section{Computational Cost}
\label{app:computational_cost}
All experiments are conducted on a single NVIDIA L4 GPU.
On synthetic datasets ($n{=}30{,}000$), training the
normalizing flow (500 epochs) takes 5--8 minutes
depending on the flow architecture (RealNVP vs.\ NSF),
while training the diffusion model (2{,}000 epochs) and
the flow matching model (2{,}000 epochs) each requires
10--15 minutes due to the additional EMA parameter
updates.
At inference, the latent-density scores require a single
deterministic forward pass per candidate and evaluate
the full calibration and test sets (${\sim}10{,}000$
points) in under $0.05$\,s.
The transport-based scores involve
$|\mathcal{T}| \times R = 120$ forward passes per
candidate ($|\mathcal{T}|{=}15$, $R{=}8$), taking
approximately $1.2$\,s for the same evaluation, roughly
$30\times$ slower but still modest in absolute terms.
The total wall-clock time for a single
train-calibrate-evaluate run, excluding volume
estimation, is under 20 minutes for all methods.
\bibliography{sample}

\end{document}